\documentclass{article}

\usepackage{fullpage}
\usepackage[utf8]{inputenc} % allow utf-8 input
\usepackage[T1]{fontenc}    % use 8-bit T1 fonts
\usepackage{hyperref}       % hyperlinks
\usepackage{url}            % simple URL typesetting
\usepackage{booktabs}       % professional-quality tables
\usepackage{amsfonts}       % blackboard math symbols
\usepackage{nicefrac}       % compact symbols for 1/2, etc.
\usepackage{microtype}      % microtypography
\usepackage{xcolor}         % colors
\usepackage{amsmath}
\usepackage{amsthm} %For proof environment
\usepackage{amssymb}
\usepackage{graphicx}
\usepackage{algorithm}
\usepackage{algpseudocode}
\usepackage{listings}
\usepackage{multirow}

\newtheorem{thm}{Theorem}%[section]
\newtheorem{theorem}[thm]{Theorem}%[section]
\newtheorem{lemma}[thm]{Lemma}
\newtheorem{cor}[thm]{Corollary}

\newtheorem{definition}[thm]{Definition}
\newtheorem{remark}[thm]{Remark}

\newtheorem{example}[thm]{Example}

\newtheorem*{exercise}{Exercise}

\newcommand\R{\mathbb{R}}

\DeclareFontFamily{U}{fsy}{}
\DeclareFontShape{U}{fsy}{m}{n}{<->s*[.9]psyr}{}
\DeclareSymbolFont{der@m}{U}{fsy}{m}{n}
\DeclareMathSymbol{\der}{\mathord}{der@m}{182}

\DeclareSymbolFont{der@m}{U}{fsy}{m}{n}
\DeclareMathSymbol{\derdelta}{\mathord}{der@m}{100}

\DeclareFontFamily{OMS}{smallo}{}
\DeclareFontShape{OMS}{smallo}{m}{n}{<->s*[.65]cmsy10}{}
\DeclareSymbolFont{smallo@m}{OMS}{smallo}{m}{n}
\DeclareMathSymbol{\smallo}{\mathord}{smallo@m}{79}

\DeclareSymbolFont{imag@m}{OT1}{cmr}{m}{ui}
\DeclareMathSymbol{\imag}{\mathord}{imag@m}{105}

\DeclareFontFamily{U}{mathb}{\hyphenchar\font45}
\DeclareFontShape{U}{mathb}{m}{n}{
      <5> <6> <7> <8> <9> <10> gen * mathb
      <10.95> matha10 <12> <14.4> <17.28> <20.74> <24.88> matha12
      }{}
\DeclareSymbolFont{mathb}{U}{mathb}{m}{n}
\DeclareMathSymbol{\monus}{2}{mathb}{"01}

\title{Stochastic Sample Approximations \\ of (Local) Moduli of Continuity}

\author{%
  Rodion Nazarov, 
  Allen Gehret, 
  Robert Noel Shorten, 
  Jakub Mareček
}

\begin{document}

\maketitle

\begin{abstract}
Modulus of local continuity is used to evaluate the robustness of neural networks and fairness of their repeated uses in closed-loop models. Here, we revisit a connection between generalized derivatives and moduli of local continuity, and present a non-uniform stochastic sample approximation for moduli of local continuity. This is of importance in studying robustness of neural networks and fairness of their repeated uses.    
\end{abstract}

\section{Introduction}

There is a substantial interest in \emph{scalable} verification of properties of neural networks.  
On one hand, some \emph{qualitative} properties of a neural network can be established  
by straightforward inspection of its compute graph.
For example, a ReLU-based network \cite{ReLU} will not be differentiable at points given by the break-points of the ReLU activation function, but will be locally Lipschitz. 
On the other hand, establishing \emph{quantitative} properties of the neural network, such as the precise Lipschitz constant and related metrics such as Consistent Robustness Analysis (CRA), is often non-trivial. 

The modulus of local continuity (e.g., the Lipschitz constant), a key quantitative property of a neural network, is important both for evaluating the robustness \cite{NIPS2016_980ecd05,limitations,yang2022closer,zuhlke2024adversarial} of neural networks and for evaluating properties  \cite{Marecek2021,10555099,nazarov2025humancompatibleinterconnecttestingpropertiesrepeated} of their repeated uses in closed-loop models. 
Recently, much of the research has focussed on two classes of methods. 
``LipMIP'' and its variants  \cite[e.g.]{tjeng2017evaluating,anderson2020strong,zhang2022general,schwan2023stability} utilized mixed-integer programming formulations of the neural network, either directly for ReLU-based networks, or using some piece-wise linear approximation \cite{GurobiML}. 
``LipSDP'' and its variants 
\cite[e.g.]{9301422,chen2020semialgebraic,latorre2020lipschitz,chen2021semialgebraic,pmlr-v115-dvijotham20a,chen2022robustness,9993136,wang2024scalability,yang2024verifying,pauli2024lipschitz,xu2024eclipse,azuma2025tight,syed2025improvedscalablelipschitzbounds}
utilize relaxations in the form of semidefinite programming (SDP). 
Substantial progress has been made recently in scaling these approaches. 
State-of the-art implementations \cite{wang2024scalability,yang2024verifying,azuma2025tight} can scale to shallow network with each of the hidden layers having hundreds of neurons. 
Still, this is substantially less than the neural networks utilized in many practical applications. 

Here, we revisit a connection between generalized derivatives and moduli of local continuity, and present a non-uniform stochastic sample approximation for these. 
At its simplest, the connection is well known 
\cite{weng2018evaluating,goodfellow2018gradientmaskingcausesclever,weng2018extensions}. Sampling values from the Clarke subdifferential at inputs sampled uniformly at random is a benchmark traditionally used for comparison of both LipSDP and LipMIP. For instance, \cite{wang2024scalability} sampled 500,000 points uniformly at random from the input space and computed the maximum norm, lower bound.
We propose algorithms for estimating the modulus of continuity using non-uniform sampling from the input space utilizing upper-confidence-bound (UCB) policies.
The resulting estimator is consistent, asymptotically unbiased, and asymptotically optimal within a class of sampling policies.

% pip install gurobipy

\section{Related Work}

\iffalse
We provide a brief overview in the supplementary material, including the treatment of 
Clarke, Bouligand, Gâteaux, Fréchet, and Hadamard generalized derivatives, sometimes also known as subdifferentials, whose elements are known as subgradients.
\fi

In this section we recall the basic theoretical setup for estimating Lipschitz constants, and review various ways this estimation has been done in practice.
We roughly follow the notation and definitions as in~\cite{jordan2020exactly}.
In particular, we assume $\|\cdot\|_{\alpha}$ and $\|\cdot\|_{\beta}$ are norms on $\R^n$ and $\R^m$ respectively.

\begin{definition}\cite[Definition 1]{jordan2020exactly}
The \textbf{local $(\alpha,\beta)$-Lipschitz constant} of a function $f:\R^n\to\R^m$ over a set $X\subseteq\R^n$ is defined as the following quantity:
\[
L^{(\alpha,\beta)}(f,X) \ := \ \sup_{x,y\in X}\frac{\|f(y)-f(x)\|_{\beta}}{\|x-y\|_{\alpha}} \quad (x\neq y)
\]
Moreover, if $L^{(\alpha,\beta)}(f,X)$ is finite, we say that $f$ is \textbf{$(\alpha,\beta)$-Lipschitz} over $X$.
\end{definition}

\noindent
The starting point for estimating Lipschitz constants in terms of Jacobians $J_f(x)$ is the following well-known fact for smooth functions defined on convex\footnote{The paper~\cite{jordan2020exactly} omits the assumption that $X$ is \emph{convex}; although this assumption gets used in the proof of~\cite[Theorem 1]{jordan2020exactly} when considering the function ``$f(x+t(y-x))$'' for arbitrary $x,y\in X$.} 
domains:
\begin{lemma}\label{lipschitz_constant_formula_smooth}
%\cite[(3)]{jordan2020exactly}
If $f:\R^n\to\R^m$ is $C^1$ and $(\alpha,\beta)$-Lipschitz over an open convex set $X$, then:
\[
L^{(\alpha,\beta)}(f,X) \ = \ \sup_{x\in X}\| J_f(x)\|_{\alpha,\beta}
\]
where $\|\cdot\|_{\alpha,\beta}$ is the induced matrix norm on $\R^{m\times n}$
\end{lemma}

In practice, many functions, including neural networks involving ReLU activations, are not smooth but are still Lipschitz.
At the points where such functions are non-smooth, one does not have a well-defined Jacobian; instead, one must resort to one of several (possibly set-valued) \emph{generalized derivatives}, for instance the \emph{Clarke Jacobian} $J_f^c$ (cf. Definition~\ref{def_Clarke_Jacobian}). In this case, Lemma~\ref{lipschitz_constant_formula_smooth} can be generalized as follows:

\begin{theorem}\cite[Theorem 1]{jordan2020exactly}\label{thm:wrong}
Suppose $f:\R^n\to\R^m$ is $(\alpha,\beta)$-lipschitz over an open convex set $X$. Then the following equality holds:
\[
L^{(\alpha,\beta)}(f,X) \ = \ \sup_{x\in X, G\in J_f^c}\|G^T\|_{\alpha,\beta}
\]
\end{theorem}

The modulus of local continuity can be estimated using a number of approaches. The most general methods may utilize finite-difference schemes, where one samples pairs of points close by, and evaluates the difference. In neural networks, one may leverage \cite{weng2018evaluating} frameworks such as PyTorch and TensorFlow to estimate generalized derivatives, usually Clarke subgradients. 
More recent methods specific to neural networks leverage the computation graph of the neural network and either integer programming \cite{tjeng2017evaluating,anderson2020strong,zhang2022general,schwan2023stability} with linear relaxations, 
or polynomial optimization \cite{chen2020semialgebraic,chen2021semialgebraic} with SDP relaxations \cite{9301422,9301422,wang2024scalability,yang2024verifying,pauli2024lipschitz,xu2024eclipse,syed2025improvedscalablelipschitzbounds}. \cite{9993136,pauli2024lipschitz,xu2024eclipse} focus on the decomposition of the large SDP variable utilizing some notion of sparsity. 
There are a variety of other approaches as well  \cite{10164809,jordan2020exactly,Bonaert2021,jafarpour2022robustness}, including 
interval methods and branch and bound without convex relaxations   \cite{bunel2018unified,9303895,wei2022certified,NEURIPS2022_1700ad4e}, but their empirical performance is broadly speaking similar. 
Our work on the non-uniform sampling is also related (but not drawing on directly) to the Adaptive Sequential Elimination of \cite{aziz2018pure}. 
%although one could also consider \cite{NEURIPS2023_3b3a83a5}. 
Approaches such as \cite{jones1993lipschitzian,Gablonsky2001} are also related, but not directly applicable, due to their assumptions of (global) Lipschitz properties.

More broadly, the estimation of robustness properties \cite{9301422} draws upon a long history of the study of neural networks in the control theory community \cite[e.g.]{1868,parisini1995receding}, and a long history of control-theoretic contributions to machine learning \cite[e.g.]{bunel2018unified}.

\section{Our Approach}

\begin{figure}[t!]
\includegraphics[width=1.0\columnwidth]{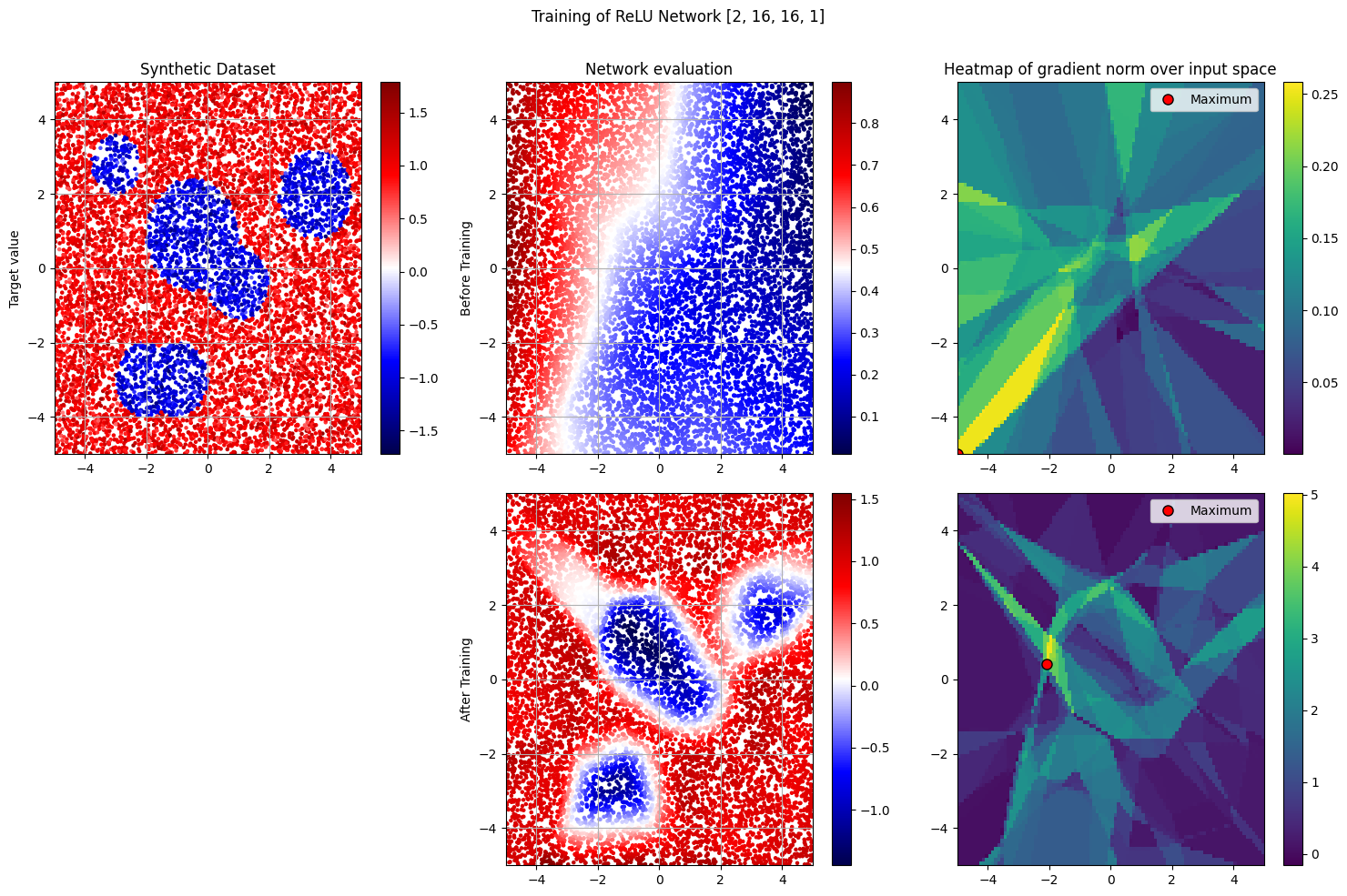}
\caption{A running example. Left:  training data. Middle: Neural network before (top) and after training (bottom). Right: Elements of Clarke subdifferential evaluated on the domain of the neural network prior to training (top) and after training (bottom). Maximum value sampled from the Clarke subdifferentials is marked with the red dot.}
\label{fig:Demo_experiment}
\end{figure}

We present a non-uniform stochastic sample approximation for estimating moduli of local continuity. First, we describe a straightforward  stochastic approximation for the modulus of (local) Lipschitz continuity \cite{weng2018evaluating} by approximating the maximum norm of the (generalized) Jacobian matrix the neural network:
$J_f(x)$. Subsequently, we extend this to non-uniform, adaptive sampling. 

\begin{remark}
Notice that many recent papers on the estimation of moduli of local continuity consider gradients and Jacobians, while they work with ReLU-based networks, which are non-smooth. 
In keeping with the literature, we will use $J_f(x)$ to denote the generalized Jacobian, computed using the Clarke subdifferential for modulus of local Lipschitz continuity, or some other generalized derivatives for other moduli of local continuity.  
\label{rem:subdifferential}
\end{remark}

\subsection{Sampling Clarke subdifferential}

Neural networks are a special case of non-smooth, non-convex functions, known as functions definable in o-minimal structures \cite{Ioffe08}. This class of functions comes with a chain rule for certain generalized derivatives, out of which the Clarke generalized derivative \cite{clarke1975generalized} is the most popular, and stability of definability, which guarantees that a composition of definable functions remains definable. (See the supplementary material.)  
This enables 
commonly used frameworks such as PyTorch and TensorFlow to have a formula for all elementary activation functions and to apply automatic differentiation of the neural network by matrix multiplication.
This way, one can evaluate elements of the Clarke subdifferential by a single call to the autograd library.  
By doing so for many inputs and looking at the maximum of the sampled values, one can estimate the modulus of local Lipschitz continuity. See Algorithm \ref{alg:SA}.

\begin{algorithm}
    \caption{Standard Stochastic Approximation, cf. \cite{weng2018evaluating} }\label{alg:SA}
    \begin{algorithmic}
        \State \textbf{Input:} Neural network $f(x)$ and its convex open domain $D$, number of samples $N$.
        \State $r \gets 0$
        \For{$1 \leq i \leq N$}
            \State Select input point $x \in D$ uniformly at random from $D$ 
            \State Evaluate $f(x)$; compute $J_f(x)$ by backpropagation, cf. Remark \ref{rem:subdifferential}
            \State Set  $r \gets \max\{r, \|J_f(x)\|_1 
 \}$
        \EndFor 
        \State \Return $r$
    \end{algorithmic}
\end{algorithm}

\subsection{Generalizations}

One can generalize Algorithm~\ref{alg:SA} in two directions: first, to other moduli of local  continuity, and second, to other forms of sampling. 
Algorithm \ref{alg:SA_static} is parametrized by the modulus of local continuity and  divides the $d$ dimensional input space into $k^d$ subregions and equally samples results from each one.
This is illustrated in Figure \ref{fig:SA_mods_partitioning}. 
Notice that the change in sampling is not a very important one, yet: Experiments of Section \ref{sec:exp} (esp. Figure \ref{fig:Equal_partitioning_error_to_budget}) show that there is no particular correlation between the number of partitions and the precision of the approximation.

\begin{algorithm}
    \caption{Generalized Stochastic Approximation with Uniform  Partitioning of the Domain}\label{alg:SA_static}
    \begin{algorithmic}
        \State \textbf{Input:} 
       Norms $\|\cdot\|_{\alpha}$ and $\|\cdot\|_{\beta}$. 
        Neural network $f(x)$ and its convex open domain $D$, number of samples $N$,
        number of divisions per dimensions $K$.
        \State $r \gets 0$
        %\State $dim \gets dimension(D)$
        \State Partion domain $D$ into subregions $S \gets $ init\_subregions$(D, K)$
        \For{$s \in S$}
            \For{$1 \leq i \leq N/|S|$}
                \State Select input point $x \in s$ uniformly at random from subregion $s \subseteq D$.
                \State Evaluate $f(x)$; compute  $\|J_f(x)\|_{\alpha,\beta}$ for a generalized Jacobian corresponding to the chosen notion of local continuity by backpropagation.Cf. Example \ref{ex:main}. 
                \State Set  $r \gets \max\{r, \|J_f(x)\|_{\alpha,\beta} 
 \}$
            \EndFor
        \EndFor
        \State \Return $r$
    \end{algorithmic}
\end{algorithm}

\subsection{Non-uniform sampling with UCB}

Non-uniform sampling, including  importance sampling   \cite{kahn1950random,robert1999monte,tokdar2010importance} and adaptive importance sampling \cite{bugallo2017adaptive}, is
a standard tool from statistics, which 
concentrates samples in regions where the estimate can be improved the most based on what has been sampled so far, without neglecting the other regions completely. 
In multi-armed bandit problems \cite{gittins2011multi,lattimore2020bandit}, the so-called upper-confidence bound policies \cite{lai1987adaptive,AuerUCB} have a long history as means of non-uniform sampling that has been shown to be asymptotically optimal in a variety of setttings. 
The infinity-armed bandit problem \cite{berry1997bandit,wang2008algorithms} can
be seen as an online sampling algorithm to estimate the maximum of a distribution.
We suggest to see the estimation of a modulus of local continuity as an infinity-armed bandit problem  \cite{berry1997bandit,wang2008algorithms}, although without the commonly considered assumptions of Lipschitz continuity \cite{Kleinberg2008,JiaYuanYu},  Bernoulli-distributed rewards \cite{berry1997bandit,NEURIPS2023_3b3a83a5}, and unique optimal arm \cite{de2021bandits}.

\begin{figure}[t!]
\includegraphics[width=1.0\columnwidth]{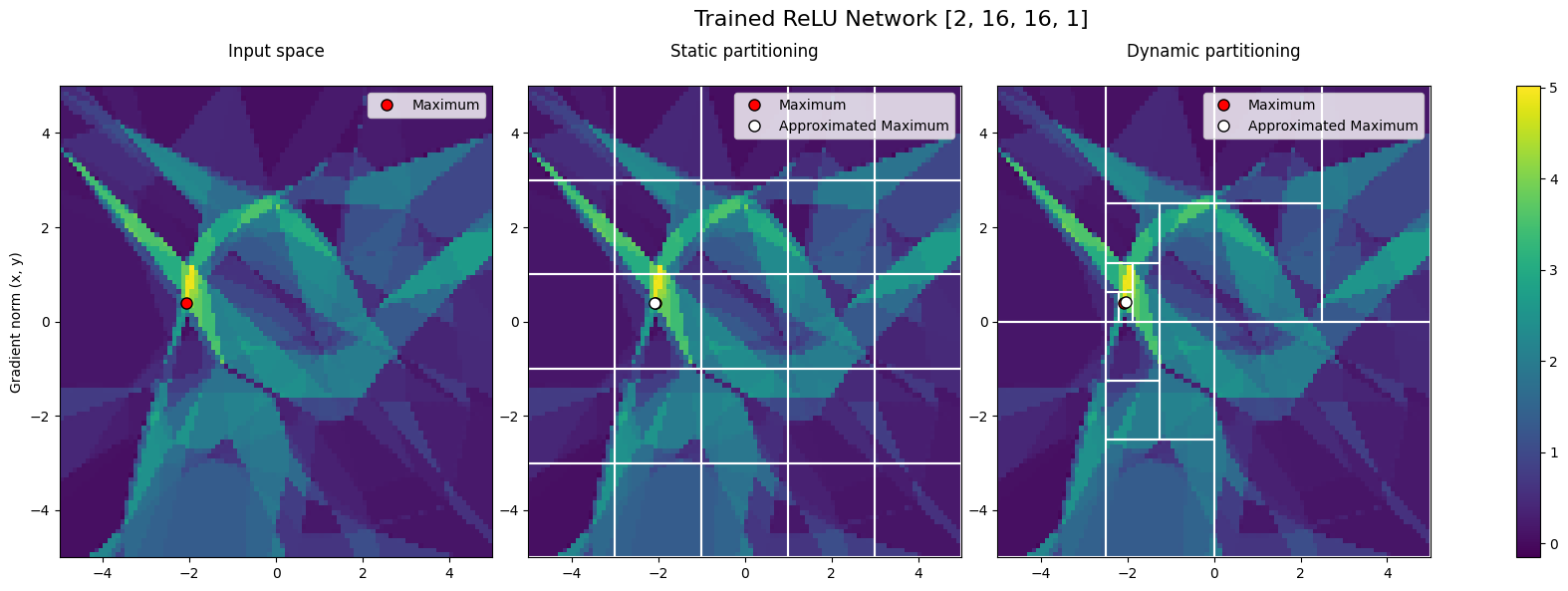}
\caption{The running example continued. Left: Elements of Clarke subdifferential evaluated on the domain of the trained neural network. Middle: uniform partitioning of the domain. Right: non-uniform partitioning of the domain utilizing the upper-confidence bounds.}
\label{fig:SA_mods_partitioning}
\end{figure}

In particular, Algorithm \ref{alg:SA_UCB} starts with an undivided domain $D$ of the neural network $f(x)$ as a single subregion. 
In each of $N$ iterations, one sample is randomly chosen from the subregion with the highest current UCB score (also known as an index), computed as: 
\[
\begin{cases}
s_{max} + c\cdot\sqrt{\frac{\ln(t+1)}{s_n}} \cdot s_{\sigma}, & \text{if } s_n > 10 \\
\infty, & \text{if } s_n \leq 10 \\
\end{cases}
\label{eq:UCB}
\]
where $s_{max}$ is the maximum encountered within the subregion so far, $s_n$ is the number of samples taken from the subregion so far, and $s_{\sigma}$ is the variance. 
After an exponentially increasing ($t_{m}$, $t_{m}^2$, $t_{m}^3, \ldots$) number of iterations,
a subregion that maximizes the UCB score \eqref{eq:UCB} is subdivided. When the domain is polyhedral, the subregions are kept polyhedral, and subdivided by axis-aligned hyperplanes. In particular, when subdividing subregion $s$, its longest side is halved by a perpendicular, axis-aligned hyperplane. $s_n = 10$ threshold is set to get initial samples for each new subregion.

\begin{algorithm}
    \caption{Generalized Stochastic Approximation with Non-uniform  Partitioning of the Domain}\label{alg:SA_UCB}
    \begin{algorithmic}
        \State \textbf{Input:} 
       Norms $\|\cdot\|_{\alpha}$ and $\|\cdot\|_{\beta}$. 
        Neural network $f(x)$ and its convex open domain $D$, number of samples $N$, UCB exploration constant $c$, subdivision time multiplier $t_{m}$.
        \State $r \gets 0$

            \State $t_{\textrm{subdivision}} \gets t_m$
        \State $S \gets \{ D \} $
        \For{$1 \leq i \leq N$}
            \State Select subregion $s \in S$ maximizing the UCB score \eqref{eq:UCB} considering exploration constant $c$. 
            \If {$i = t_{\textrm{subdivision}}$}
                \State Replace $s \in S$ in $S$ with subdivide$(s)$
                \State $t_{\textrm{subdivision}} \gets t_{\textrm{subdivision}} \cdot t_{m}$
            \EndIf    
                
        %\State Partion domain $D$ into subregions $S \gets $ init\_subregions$(D, K)$
        %\For{$s \in S$}
        %    \For{$1 \leq i \leq N/|S|$}
          
                \State Select input point $x \in s$ uniformly at random from subregion $s \subseteq D$.
                \State Evaluate $f(x)$; compute  $\|J_f(x)\|_{\alpha,\beta}$ for a generalized Jacobian corresponding to the chosen notion of local continuity by backpropagation. Cf. Example \ref{ex:main}.
                \State Set  $r \gets \max\{r, \|J_f(x)\|_{\alpha,\beta} 
 \}$
        \EndFor
        \State \Return $r$
    \end{algorithmic}
\end{algorithm}

% Stochastic Approximation: A Dynamical Systems Viewpoint

\section{An Analysis}
%\textcolor{red}{Anthony to do some proofs of consistency under any distribution of the samples whose support is the whole domain.}

It is clear by Theorem~\ref{thm:wrong} that Algorithms 1-3 will always return a lower bound for the Lipschitz constant provided that ``$J_f(x)$'' is contained in the Clarke Jacobian $J_f^c(x)$. Our analysis addresses the following two questions:
\begin{enumerate}
\item What degree of flexibility do we have in replacing $J_f^c$ in Theorem~\ref{thm:wrong} with some other set-valued ``generalized derivative''?
\item What are the statistical properties of Algorithm 3?
\end{enumerate}

\subsection{Modifying Theorem~\ref{thm:wrong}}
It has been observed many times~\cite{davis2020stochastic,bolte2021conservative} that nearly all neural networks used in practice are definable in some o-minimal structure. Under this assumption, we describe in Theorem~\ref{thm:main} below to what extent the role of the Clarke Jacobian in Theorem~\ref{thm:wrong} can be replaced with a suitable alternative.  See Appendix~\ref{appendix_Supplementary_Material} for an elaboration of the current subsection, including a proof of Theorem~\ref{thm:main}.

We make the following assumptions:
\begin{itemize}
\item $\frak{R}=(\R;<,+,\cdot,\ldots)$ is an o-minimal expansion of the real field and ``definable'' means ``definable in $\frak{R}$ with parameters''. The reader is free to focus on the special case where $\frak{R}=(\R;<,+,\cdot)$ is the usual real field, in which case ``definable'' is a synonym for ``semialgebraic''
\item $\|\cdot\|_{\alpha}$ and $\|\cdot\|_{\beta}$ are norms on $\R^n$ and $\R^m$ respectively, not necessarily definable (in the sense of $\frak{R}$)
\item $X\subseteq\R^n$ is a definable open convex set
\item $f:X\to\R^m$ is a definable $(\alpha,\beta)$-Lipschitz function
\end{itemize}

We introduce the following provisional terminology:

\begin{definition}
We say a set-valued map $D:X\rightrightarrows\R^{m\times n}$ is \textbf{good} for $f$ if:
\begin{itemize}
\item $D(x)\subseteq J_f^c(x)$ for every $x\in X$, and
\item there exists $X_0\subseteq X$ such that:
\begin{itemize}
\item $X\setminus X_0$ is a null set with respect to Lebesgue measure on $\R^n$, and
\item $D(x)$ is nonempty for every $x\in X_0$
\end{itemize}
\end{itemize}
[Note that we do not require $D$ or $X_0$ to be definable.]
\end{definition}

Under these assumptions, we have the following variation of Theorem~\ref{thm:wrong}:

\begin{theorem}
\label{thm:main}
If $D$ is good for $f$, then:
\[
L^{(\alpha,\beta)}(f,X) \ = \ \sup_{x\in X, G\in D(x)}\|G\|_{\alpha,\beta}
\]
\end{theorem}

\begin{example}
Apart from $D=J_f^c$, here are some examples of $D$ satisfying conditions of Theorem \ref{thm:main}:
\begin{itemize}
\item $D$ can be the partially-defined single-valued map returning $J_f(x)$ at all points $x\in X$ for which $f$ is Fr\'{e}chet-differentiable; more generally, given $r\geq 1$, $D$ can return $J_f(x)$ at all points $x\in X$ for which $f$ is $C^r$-smooth
\item $D$ can be an arbitrary selection of $J_f^c$ defined a.e. in $X$
\item when $m=1$, $D$ can be the Fr\'{e}chet, limiting, or Clarke subdifferential; e.g.,~\cite{bolte2007clarke}. 
\end{itemize}
\label{ex:main}
\end{example}

Note that Theorem~\ref{thm:main} can fail if we do not assume that $f$ is definable:

\begin{example}
Let $f:\R\to\R$ be a Lipschitz function with Lipschitz constant $1$ such that $J_f^c(x)=[-1,1]$ for every $x\in \R$; such functions occur generically according to~\cite{Ioffe08,borwein1997distinct} although they cannot be definable (as this would violate the Smooth Monotonicity Theorem~\cite[7.2.5]{van1998tame}). Set $D(x):=\{0\}$ for every $x\in\R$; then $D$ is good for $f$, however $\sup_{x\in\R,G\in D(x)}\|G\|_{\alpha,\beta}=0\neq 1$.
\end{example}

As our idealized model for ``$J_f(x)$'' gets closer to how automatic differentiation (AD) is actually implemented by PyTorch or TensorFlow, then Theorem~\ref{thm:main} may no longer be true due to sporadic behavior on null sets:

\begin{remark}
Suppose $D$ is either (1) the output of AD as modeled by selection Jacobians~\cite{bolte2020mathematical}, or (2) a conservative mapping for $f$~\cite{bolte2021conservative}. Then $D=\{J_f(x)\}$ a.e. on $X$ and so $L^{(\alpha,\beta)}(f,X)\leq\sup_{x\in X,G\in D(x)}\|G\|_{\alpha,\beta}$. However, equality can fail in this case as~\cite{bolte2020mathematical,bolte2021conservative} provide easy examples where $f\equiv 0$, although $D(x)$ can contain nonzero vectors.
\end{remark}

%\subsection{Statistical properties of Algorithm 3}
The consistency of the estimate obtained by Algorithm \ref{alg:SA_UCB} is clear from the fact that we are sampling from any point in the domain $D$ with a positive probability and the strong law of large numbers \cite[cf. Theorem 4.7.3]{robert1999monte}. 
The estimate is not unbiased, but only 
asymptotically unbiased 
\cite{robert1999monte}.
The asymptotic optimality of Algorithm \ref{alg:SA_UCB} within a certain class of sampling schemes is less trivial, but follows from well-known analyses \cite{NIPS2011_7e889fb7} of UCB policies for the infinity-armed bandit problems.

\begin{figure}[t!]
\includegraphics[width=1.0\columnwidth]{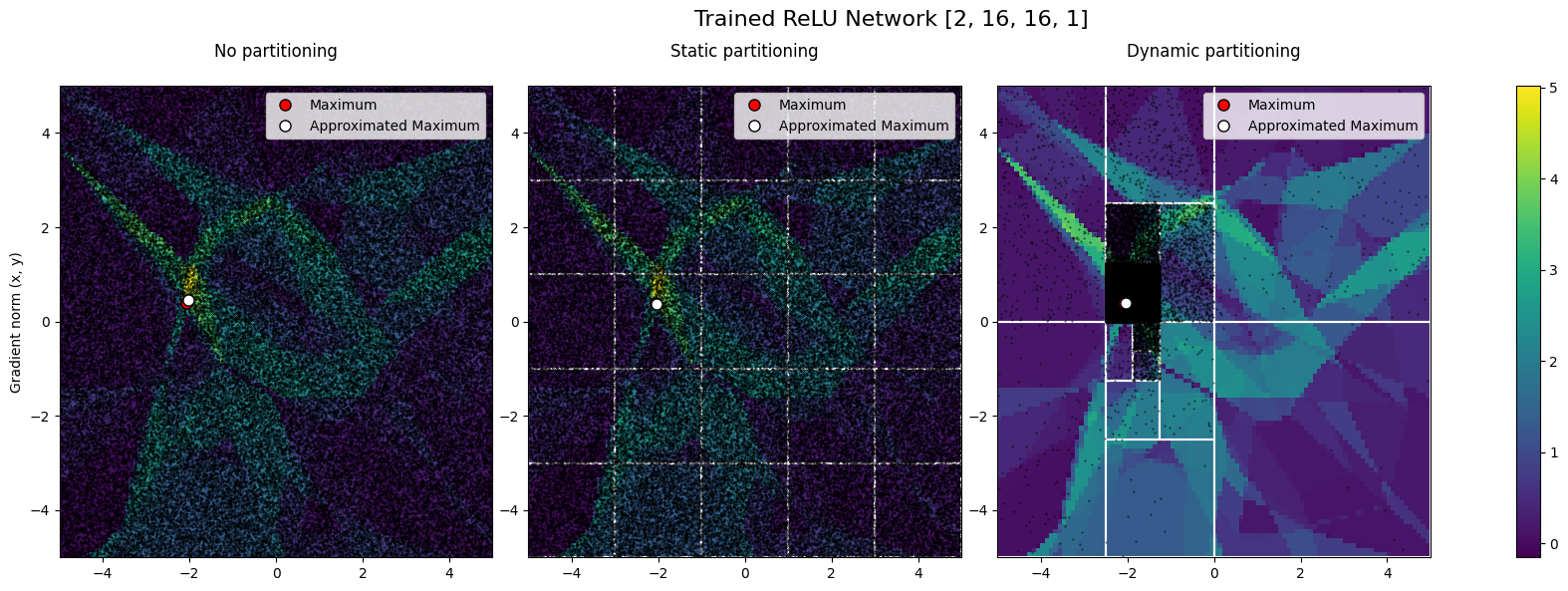}
\caption{The running example continued. Left: Samples obtained using the standard stochastic approximation (Algorithm~\ref{alg:SA}). Middle: samples obtained using uniform partitioning of the domain (Algorithm~\ref{alg:SA_static}). Right: samples obtained using non-uniform partitioning of the domain utilizing the upper-confidence bounds (Algorithm~\ref{alg:SA_UCB}).}
\label{fig:SA_mods_sampling}
\end{figure}

\begin{figure}[t!]
\includegraphics[width=0.45\columnwidth]{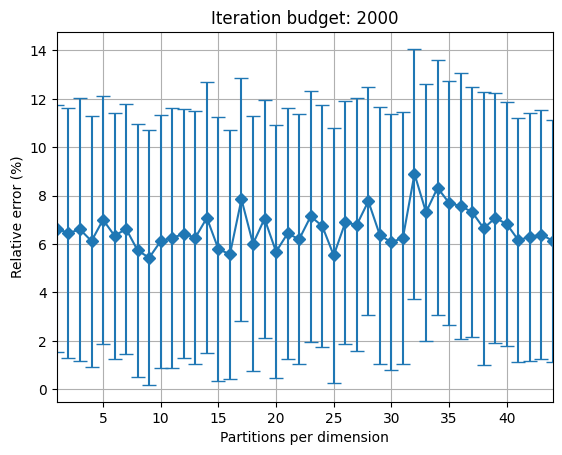}
\includegraphics[width=0.45\columnwidth]{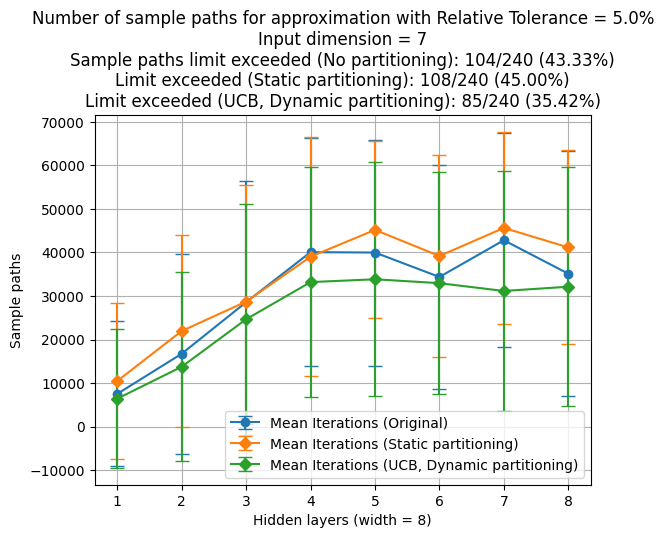}
\caption{Illustrative results. Left: relative error of Algorithm \ref{alg:SA_static} is robust with respect to the number of subregions. Right: performance of Algorithms \ref{alg:SA} and \ref{alg:SA_static} is very similar, when considering one and the same subdiffrential. Algorithm \ref{alg:SA_UCB} strictly improves upon both Algorithms \ref{alg:SA} and \ref{alg:SA_static}.}
\label{fig:Equal_partitioning_error_to_budget}
\end{figure}

\section{Experimental results}
\label{sec:exp}

As a proof of concept, we have implemented our approach in PyTorch using their automatic differentiation engine (torch.autograd). 
In our experiments, we have used ReLU-based neural nets (composed of torch.nn.Linear and torch.nn.ReLU layers). To describe their dimensions, we use the notation such as of $[2, w_1, w_2, 1]$, where there are two scalars on the inputs, a linear layer with $w_1$ output features, ReLU activation,  a linear layer with $w_1$ input features and $w_2$ output features, ReLU activation, and a linear layer with $w_2$ input features and 1 output feature. For the purposes of visualization, we often focus on the two-dimensional input. 
%An example of a net described as [2, 128, 128, 1] would be:
%\begin{lstlisting}
%ReLUNet(
%  (net): Sequential(
%    (1): Linear(in_features=2, out_features=128, bias=True)
%    (2): ReLU()
%    (3): Linear(in_features=128, out_features=128, bias=True)
%    (4): ReLU()
%    (5): Linear(in_features=128, out_features=1, bias=True)
%  )
%)
%\end{lstlisting}
We generate synthetic datasets to match, one of which is displayed on the left in Figure~\ref{fig:Demo_experiment}, with points randomly generated hyperspheres with points generated randomly with mean $-1.0$, while the remainder of the hypercube is filled with values generated at random with mean $1.0$.

In Figures \ref{fig:Demo_experiment}--\ref{fig:SA_mods_sampling}, we focus on the running-example neural network [2, 16, 16, 1] presented in Figure \ref{fig:Demo_experiment}. 

In Figure \ref{fig:Equal_partitioning_error_to_budget} (on the right), we summarize the experiments where all the generated datasets had 3 hyperspheres of random radius $0.1 \times r_{domain} \leq r_{sphere} \leq 0.4 \times r_{domain}$ randomly centered on the input space. 
For each number of hidden layers in the right plot in Figure \ref{fig:Equal_partitioning_error_to_budget}, 30 test runs were evaluated, for each test run a new neural net was trained on a random data set with 800 points. Training was performed with torch.nn.MSELoss() as loss function, torch.optim.Adam() as optimizer with learning rate $5 \times 10^{-4}$ for 500 epochs. 
All 3 algorithms were limited to 60000 evaluations, For Algorithm \ref{alg:SA_static}, 2 divisions per dimension were set, resulting in $2^7$ subregions; for Algorithm \ref{alg:SA_UCB} the exploration parameter $c$ (\ref{eq:UCB}) was set to 10 and the subdivision time multiplier was set to 2.\

In Figure \ref{fig:Equal_partitioning_error_to_budget} (on the left),
we ran Algorithm \ref{alg:SA_static} with different partition parameter on the one trained neural net of depth 4 and width 8, for each parameter value 100 runs were evaluated. 
The 2D heat maps of the gradients that are present in the figures were generated by evaluating the $400^2$ grid. The red dot in the plots presented as Maximum is within the $0.001\%$ error of the actual value computed by LipMIP.

In Figure \ref{fig:SA_scale_time}, we evaluate using untrained neural nets. Stochastic Approximation and LipSDP were run without parallelization, and 5 cores were allocated for Gurobi engine in LipMIP.

In Table \ref{table:BMnist}, 
we evaluate using neural nets trained on binary MNIST (classification between 1 and 7),
following the benchmarking procedure suggested by \cite{tjeng2017evaluating}. 
Column ``setup'' presents the dimensions of the net, and acceptable error of Algorithms \ref{alg:SA} and \ref{alg:SA_UCB} (when applicable). For each net, we ran 5 evaluations on trained net and best result in terms of relative error for each method was chosen. For the last two experiments, run time of LipMIP \cite{tjeng2017evaluating} was limited. Considering that LipSDP produces an upper bound, the lowest (best) result is presented. For Algorithms \ref{alg:SA} and \ref{alg:SA_UCB}, the highest result is presented. Rows 2 and 4 present experiments, where Algorithms \ref{alg:SA} and \ref{alg:SA_UCB} could compare the outcomes against the LipMIP result. (This is of interest, because one may often have many  similar nets and may wish to  estimate of their moduli of continuity one by one, using as few iterations as possible.) Because LipMIP computes the modulus of $\langle c, f\rangle$, modified versions of Algorithm \ref{alg:SA_UCB} were used, where the modulus of local continuity was approximated from $\langle c, f\rangle$, rather than from the whole output of $f$.
Rows 6 and 7 present experiments, where the scalability of Algorithm  \ref{alg:SA_UCB} strictly exceeds the scalability of LipMIP.
Rows 6 and 7 present experiments, where the scalability of Algorithm  \ref{alg:SA_UCB} strictly exceeds the scalability of LipMIP.
By ``?'' in the ``Relative error'' column, we indicate that we do not have the exact modulus of local continuity. The comparison against LipSDP shows that our estimates improve upon LipSDP and its variants, whose upper bound is very loose (at least 356\% higher in row 7), without requiring a longer run-time.

\begin{figure}[t!]
\includegraphics[width=0.45\columnwidth]{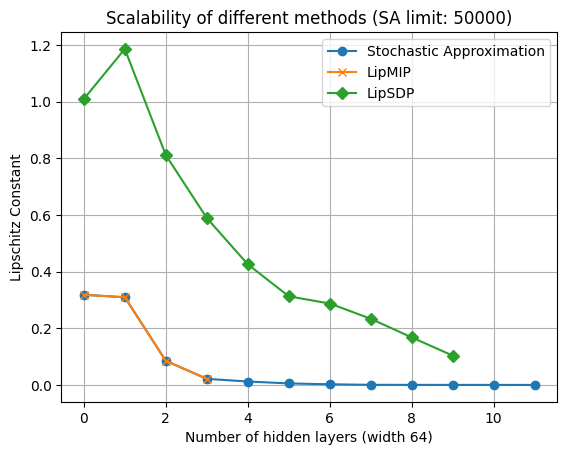}
\includegraphics[width=0.45\columnwidth]{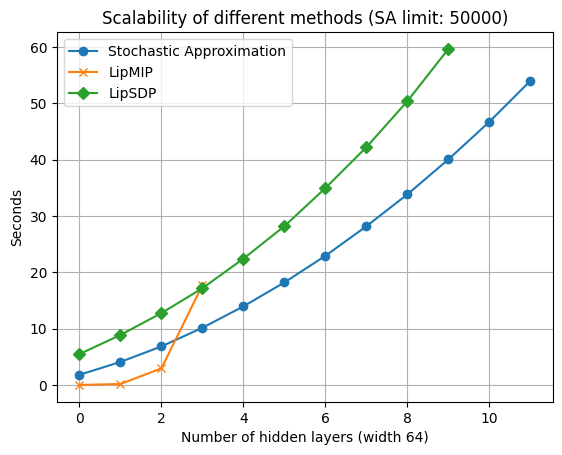}
\caption{Illustrative comparison against LipMIP and LipSDP. 
Left: with the number of hidden layers (each of width 64), the modulus of local Lipschitz continuity decreases sharply. LipSDP overestimates the modulus by a substantial margin, while stochastic approximation with 50000 samples tracks the true values computed using LipMIP. 
Right: within a given run-time limit of 60 seconds, Algorithm \ref{alg:SA_UCB} with 50000 samples scales to depth-11 networks (704 neurons), where LipSDP utilizing Mosek scales to depth 9 (576 neurons) and LipMIP utilizing Gurobi scales to depth 3 (192 neurons).}
\label{fig:SA_scale_time}
\end{figure}

\begin{table}[t!]
\centering
\caption{Comparison of methods on Binary MNIST  following the benchmarking procedure suggested in the LipMIP paper \cite{tjeng2017evaluating}. }
\label{table:BMnist}
\begin{tabular}{|c|c c c c|}
\hline\hline
Setup & Method &  Value & Time (s) & Relative Error (\%) \\
\hline
\multirow{4}{*}{[784, 8, 2]}
    & LipMIP & 77.08519 & 1.523 & 0 \\
    & LipSDP & 100.76831 & 7.610 & +30.7 \\
    & Algorithm \ref{alg:SA} (5000) & 77.07964 & 0.312 & -0.007 \\
    & Algorithm \ref{alg:SA_UCB} (5000) & 77.07964 & 1.330 & -0.007 \\
\hline
\multirow{4}{*}{\shortstack{[784, 8, 2] \\ Approximation  tolerance 1\%}}
    & LipMIP & 77.08519 & 1.523 & 0 \\
    & LipSDP & 100.76831 & 7.610 & +30.7 \\
    & Algorithm \ref{alg:SA} (5000) & 77.07964 & 0.0132 & -0.007 \\
    & Algorithm \ref{alg:SA_UCB} (5000) & 77.07964 & 0.0004 & -0.007 \\
\hline
\multirow{4}{*}{[784, 8, 8, 2]}
    & LipMIP & 79.24930 & 6.214 & 0 \\
    & LipSDP & 104.73195 & 8.987 & +32.155 \\
    & Algorithm \ref{alg:SA} (10000) & 79.01881 & 0.776 & -0.291 \\
    & Algorithm \ref{alg:SA_UCB} (10000) & 79.01880 & 2.696 & -0.291 \\
\hline
\multirow{4}{*}{\shortstack{[784, 8, 8, 2] \\ Approximation tolerance 1\%}}
    & LipMIP & 79.24930 & 6.214 & 0 \\
    & LipSDP & 104.73195 & 8.987 & +32.155 \\
    & Algorithm \ref{alg:SA} (10000) & 79.01881 & 0.051 & -0.291 \\
    & Algorithm \ref{alg:SA_UCB} (10000) & 79.01880 & 0.408 & -0.291 \\
\hline
\multirow{4}{*}{[784, 20, 20, 2]}
    & LipMIP & 349.67096 & 6.048 & 0 \\
    & LipSDP & 479.54083 & 12.48653 & +37.140 \\
    & Algorithm \ref{alg:SA} (150000) & 277.22370 & 11.931 & -20.719 \\
    & Algorithm \ref{alg:SA_UCB} (150000) & 290.12729 & 44.975 & -17.028 \\
\hline
\multirow{4}{*}{[784, 8, 8, 8, 2]}
    & LipMIP (timeout) & 114.48170 & 120.0 & ? \\
    & LipSDP & 125.98295 & 8.987 & ? \\
    & Algorithm \ref{alg:SA} (100000) & 89.87323 & 9.255 & ? \\
    & Algorithm \ref{alg:SA_UCB} (100000) & 89.87323 & 20.522 & ? \\
\hline 
\multirow{4}{*}{[784, 8, 8, 8, 8, 8, 2]}
    & LipMIP (timeout) & 519.65584 & 300.0 & ? \\
    & LipSDP & 429.07478 & 11.472 & ? \\
    & Algorithm \ref{alg:SA} (100000) & 94.14903 & 12.417 & ? \\
    & Algorithm \ref{alg:SA_UCB} (100000) & 94.14903 & 24.549 & ? \\
\hline 
\end{tabular}
\end{table}

\section{Conclusions and Limitations}

We have elaborated upon the connections between ``generalized derivatives'' and estimating Lipschitz constants in the setting of \emph{tame optimization} \cite{Ioffe08}. 

We have also presented non-uniform sampling within the context of estimating the modulus of continuity using, and showed that it improves the scalability of the stochastic sample approximations therein, improving also over some well-established methods not based on stochastic sample approximations (LipMIP, LipSDP).
The advantage of non-uniform sampling (Algorithm \ref{alg:SA_UCB}) is clear especially when there are both areas of vanishing (sub)gradients 
\cite{hochreiter1998vanishing} and areas with high variance of the subgradients in the domain. This is common \cite{hanin2018neural} such as in the case of deeper and wider nets.

A key limitation of our method, as well as any other method for the verification of properties of neural networks known presently, is scalability. While our algorithms 
can be applied to a neural network in any dimension, and 
our results (e.g., Table~
\ref{table:BMnist} and Figure \ref{fig:SA_scale_time}) suggest some improvement in run-time for instances where all methods are applicable, this is far from the scale of foundational models utilized in many applications. 
One may wish to bound the rates of convergence of the UCB-based policies and to improve the per-iteration run-time. 
In order to improve the per-iteration run-time (the number of sample paths per second), one may wish to implement the approach in JAX or Enzyme \cite{NEURIPS2020_9332c513}. %which would make it harder to use with neural networks implemented in PyTorch, though.  

The performance of the Algorithm \ref{alg:SA_UCB} is also affected by the choice of the UCB exploration parameter $c$. Poor choice may lead to a higher number of iterations required to find a close approximation because the subregion of the input domain that contains the result may not be explored for a long time. It is not clear what value to choose without a general idea of the gradient space. A possible way to deal with this is to sample a small amount of gradients before running Algorithm \ref{alg:SA_UCB} to get an idea of the mean and standard deviation of the gradients across the input domain and then choose $c$ accordingly.

Related methods have extensive applications, e.g., in the safety filters for AI systems, and 
in the design of controllers %\cite{kenefake2024multiparametric,wang2024monotone} 
for such closed loops with guarantees of ergodic behaviour.

%\begin{ack}
%Use unnumbered first level headings for the acknowledgments. All acknowledgments
%go at the end of the paper before the list of references. Moreover, you are required to declare
%funding (financial activities supporting the submitted work) and competing interests (related financial activities outside the submitted work).
%More information about this disclosure can be found at: \url{https://neurips.cc/Conferences/2025/PaperInformation/FundingDisclosure}.

%o {\bf not} include this section in the anonymized submission, only in the final paper. You can use the \texttt{ack} environment provided in the style file to automatically hide this section in the anonymized submission.
%\end{ack}

\bibliographystyle{ieeetr}
\bibliography{main}

\begin{thebibliography}{10}

\bibitem{cybenko1989approximation}
G.~Cybenko, ``Approximation by superpositions of a sigmoidal function,'' {\em
  Mathematics of control, signals and systems}, vol.~2, no.~4, pp.~303--314,
  1989.

\bibitem{barron1993universal}
A.~R. Barron, ``Universal approximation bounds for superpositions of a
  sigmoidal function,'' {\em IEEE Transactions on Information theory}, vol.~39,
  no.~3, pp.~930--945, 1993.

\bibitem{ReLU}
A.~S. Householder, ``A theory of steady-state activity in nerve-fiber networks:
  I. definitions and preliminary lemmas,'' {\em The bulletin of mathematical
  biophysics}, vol.~3, no.~2, pp.~63--69, 1941.

\bibitem{NIPS2016_980ecd05}
O.~Bastani, Y.~Ioannou, L.~Lampropoulos, D.~Vytiniotis, A.~Nori, and
  A.~Criminisi, ``Measuring neural net robustness with constraints,'' in {\em
  Advances in Neural Information Processing Systems} (D.~Lee, M.~Sugiyama,
  U.~Luxburg, I.~Guyon, and R.~Garnett, eds.), vol.~29, Curran Associates,
  Inc., 2016.

\bibitem{limitations}
T.~Huster, C.-Y.~J. Chiang, and R.~Chadha, ``Limitations of the lipschitz
  constant as a defense against adversarial examples,'' in {\em ECML PKDD 2018
  Workshops: Nemesis 2018, UrbReas 2018, SoGood 2018, IWAISe 2018, and Green
  Data Mining 2018, Dublin, Ireland, September 10-14, 2018, Proceedings},
  (Berlin, Heidelberg), p.~16–29, Springer-Verlag, 2018.

\bibitem{yang2022closer}
Z.~Yang, T.~Pang, and Y.~Liu, ``A closer look at the adversarial robustness of
  deep equilibrium models,'' {\em Advances in Neural Information Processing
  Systems}, vol.~35, pp.~10448--10461, 2022.

\bibitem{zuhlke2024adversarial}
M.-M. Z{\"u}hlke and D.~Kudenko, ``Adversarial robustness of neural networks
  from the perspective of lipschitz calculus: A survey,'' {\em ACM Computing
  Surveys}, 2024.

\bibitem{Marecek2021}
J.~Mare{\v{c}}ek, M.~Roubalik, R.~Ghosh, R.~N. Shorten, and F.~R. Wirth,
  ``Predictability and fairness in load aggregation and operations of virtual
  power plants,'' {\em Automatica}, vol.~147, p.~110743, 2023.

\bibitem{10555099}
Q.~Zhou, R.~Ghosh, R.~Shorten, and J.~Mareček, ``Closed-loop view of the
  regulation of ai: Equal impact across repeated interactions,'' in {\em 2024
  IEEE 40th International Conference on Data Engineering Workshops (ICDEW)},
  pp.~176--181, 2024.

\bibitem{nazarov2025humancompatibleinterconnecttestingpropertiesrepeated}
R.~Nazarov, A.~Quinn, R.~Shorten, and J.~Marecek,
  ``humancompatible.interconnect: Testing properties of repeated uses of
  interconnections of ai systems,'' 2025.

\bibitem{tjeng2017evaluating}
V.~Tjeng, K.~Y. Xiao, and R.~Tedrake, ``Evaluating robustness of neural
  networks with mixed integer programming,'' in {\em International Conference
  on Learning Representations}, 2019.

\bibitem{anderson2020strong}
R.~Anderson, J.~Huchette, W.~Ma, C.~Tjandraatmadja, and J.~P. Vielma, ``Strong
  mixed-integer programming formulations for trained neural networks,'' {\em
  Mathematical Programming}, vol.~183, no.~1, pp.~3--39, 2020.

\bibitem{zhang2022general}
H.~Zhang, S.~Wang, K.~Xu, L.~Li, B.~Li, S.~Jana, C.-J. Hsieh, and J.~Z. Kolter,
  ``General cutting planes for bound-propagation-based neural network
  verification,'' {\em Advances in neural information processing systems},
  vol.~35, pp.~1656--1670, 2022.

\bibitem{schwan2023stability}
R.~Schwan, C.~N. Jones, and D.~Kuhn, ``Stability verification of neural network
  controllers using mixed-integer programming,'' {\em IEEE Transactions on
  Automatic Control}, vol.~68, no.~12, pp.~7514--7529, 2023.

\bibitem{GurobiML}
{Gurobi Optimization, LLC}, ``Gurobi machine learning,''
\newblock
  https://gurobi-machinelearning.readthedocs.io/en/stable/userguide.html.

\bibitem{9301422}
M.~Fazlyab, M.~Morari, and G.~J. Pappas, ``Safety verification and robustness
  analysis of neural networks via quadratic constraints and semidefinite
  programming,'' {\em IEEE Transactions on Automatic Control}, vol.~67, no.~1,
  pp.~1--15, 2022.

\bibitem{chen2020semialgebraic}
T.~Chen, J.~B. Lasserre, V.~Magron, and E.~Pauwels, ``Semialgebraic
  optimization for lipschitz constants of relu networks,'' {\em Advances in
  Neural Information Processing Systems}, vol.~33, pp.~19189--19200, 2020.

\bibitem{latorre2020lipschitz}
F.~Latorre, P.~Rolland, and V.~Cevher, ``Lipschitz constant estimation of
  neural networks via sparse polynomial optimization,'' in {\em International
  Conference on Learning Representations}, 2020.

\bibitem{chen2021semialgebraic}
T.~Chen, J.~B. Lasserre, V.~Magron, and E.~Pauwels, ``Semialgebraic
  representation of monotone deep equilibrium models and applications to
  certification,'' {\em Advances in Neural Information Processing Systems},
  vol.~34, pp.~27146--27159, 2021.

\bibitem{pmlr-v115-dvijotham20a}
K.~D. Dvijotham, R.~Stanforth, S.~Gowal, C.~Qin, S.~De, and P.~Kohli,
  ``Efficient neural network verification with exactness characterization,'' in
  {\em Proceedings of The 35th Uncertainty in Artificial Intelligence
  Conference} (R.~P. Adams and V.~Gogate, eds.), vol.~115 of {\em Proceedings
  of Machine Learning Research}, pp.~497--507, PMLR, 22--25 Jul 2020.

\bibitem{chen2022robustness}
T.~Chen, {\em Robustness verification of neural networks using polynomial
  optimization}.
\newblock PhD thesis, Universit{\'e} Paul Sabatier-Toulouse III, 2022.

\bibitem{9993136}
A.~Xue, L.~Lindemann, A.~Robey, H.~Hassani, G.~J. Pappas, and R.~Alur,
  ``Chordal sparsity for lipschitz constant estimation of deep neural
  networks,'' in {\em 2022 IEEE 61st Conference on Decision and Control (CDC)},
  pp.~3389--3396, 2022.

\bibitem{wang2024scalability}
Z.~Wang, B.~Hu, A.~J. Havens, A.~Araujo, Y.~Zheng, Y.~Chen, and S.~Jha, ``On
  the scalability and memory efficiency of semidefinite programs for lipschitz
  constant estimation of neural networks,'' in {\em The Twelfth International
  Conference on Learning Representations}, 2024.

\bibitem{yang2024verifying}
J.~Yang, S.~{\DH}ura{\v{s}}inovi{\'c}, J.-B. Lasserre, V.~Magron, and J.~Zhao,
  ``Verifying properties of binary neural networks using sparse polynomial
  optimization,'' {\em arXiv preprint arXiv:2405.17049}, 2024.

\bibitem{pauli2024lipschitz}
P.~Pauli, D.~Gramlich, and F.~Allg{\"o}wer, ``Lipschitz constant estimation for
  general neural network architectures using control tools,'' {\em arXiv
  preprint arXiv:2405.01125}, 2024.

\bibitem{xu2024eclipse}
Y.~Xu and S.~Sivaranjani, ``{ECL}ipse: Efficient compositional lipschitz
  constant estimation for deep neural networks,'' in {\em The Thirty-eighth
  Annual Conference on Neural Information Processing Systems}, 2024.

\bibitem{azuma2025tight}
J.~Lan, Y.~Zheng, and A.~Lomuscio, ``Tight neural network verification via
  semidefinite relaxations and linear reformulations,'' in {\em AAAI},
  pp.~7272--7280, 2022.

\bibitem{syed2025improvedscalablelipschitzbounds}
U.~Syed and B.~Hu, ``Improved scalable lipschitz bounds for deep neural
  networks,'' 2025.

\bibitem{weng2018evaluating}
T.-W. Weng, H.~Zhang, P.-Y. Chen, J.~Yi, D.~Su, Y.~Gao, C.-J. Hsieh, and
  L.~Daniel, ``Evaluating the robustness of neural networks: An extreme value
  theory approach,'' in {\em International Conference on Learning
  Representations}, 2018.

\bibitem{goodfellow2018gradientmaskingcausesclever}
I.~Goodfellow, ``Gradient masking causes clever to overestimate adversarial
  perturbation size,'' 2018.

\bibitem{weng2018extensions}
T.-W. Weng, H.~Zhang, P.-Y. Chen, A.~Lozano, C.-J. Hsieh, and L.~Daniel, ``On
  extensions of clever: A neural network robustness evaluation algorithm,'' in
  {\em 2018 IEEE Global Conference on Signal and Information Processing
  (GlobalSIP)}, pp.~1159--1163, IEEE, 2018.

\bibitem{jordan2020exactly}
M.~Jordan and A.~G. Dimakis, ``Exactly computing the local lipschitz constant
  of relu networks,'' {\em Advances in Neural Information Processing Systems},
  vol.~33, pp.~7344--7353, 2020.

\bibitem{10164809}
T.~Avant and K.~A. Morgansen, ``Analytical bounds on the local lipschitz
  constants of relu networks,'' {\em IEEE Transactions on Neural Networks and
  Learning Systems}, vol.~35, no.~10, pp.~13902--13913, 2024.

\bibitem{Bonaert2021}
G.~Bonaert, D.~I. Dimitrov, M.~Baader, and M.~Vechev, ``Fast and precise
  certification of transformers,'' in {\em Proceedings of the 42nd ACM SIGPLAN
  International Conference on Programming Language Design and Implementation},
  PLDI 2021, (New York, NY, USA), p.~466–481, Association for Computing
  Machinery, 2021.

\bibitem{jafarpour2022robustness}
S.~Jafarpour, M.~Abate, A.~Davydov, F.~Bullo, and S.~Coogan, ``Robustness
  certificates for implicit neural networks: A mixed monotone contractive
  approach,'' in {\em Learning for Dynamics and Control Conference},
  pp.~917--930, PMLR, 2022.

\bibitem{bunel2018unified}
R.~R. Bunel, I.~Turkaslan, P.~Torr, P.~Kohli, and P.~K. Mudigonda, ``A unified
  view of piecewise linear neural network verification,'' {\em Advances in
  Neural Information Processing Systems}, vol.~31, 2018.

\bibitem{9303895}
B.~Karg and S.~Lucia, ``Stability and feasibility of neural network-based
  controllers via output range analysis,'' in {\em 2020 59th IEEE Conference on
  Decision and Control (CDC)}, pp.~4947--4954, 2020.

\bibitem{wei2022certified}
C.~Wei and J.~Z. Kolter, ``Certified robustness for deep equilibrium models via
  interval bound propagation,'' in {\em International Conference on Learning
  Representations}, 2022.

\bibitem{NEURIPS2022_1700ad4e}
E.~Abad~Rocamora, M.~F. Sahin, F.~Liu, G.~Chrysos, and V.~Cevher, ``Sound and
  complete verification of polynomial networks,'' in {\em Advances in Neural
  Information Processing Systems} (S.~Koyejo, S.~Mohamed, A.~Agarwal,
  D.~Belgrave, K.~Cho, and A.~Oh, eds.), vol.~35, pp.~3517--3529, Curran
  Associates, Inc., 2022.

\bibitem{aziz2018pure}
M.~Aziz, J.~Anderton, E.~Kaufmann, and J.~Aslam, ``Pure exploration in
  infinitely-armed bandit models with fixed-confidence,'' in {\em Algorithmic
  Learning Theory}, pp.~3--24, PMLR, 2018.

\bibitem{jones1993lipschitzian}
D.~R. Jones, C.~D. Perttunen, and B.~E. Stuckman, ``Lipschitzian optimization
  without the lipschitz constant,'' {\em Journal of optimization Theory and
  Applications}, vol.~79, pp.~157--181, 1993.

\bibitem{Gablonsky2001}
J.~M.~X. Gablonsky and C.~T. Kelley, {\em Modifications of the direct
  algorithm}.
\newblock PhD thesis, North Carolina State University, Raleigh, North Carolina,
  2001.

\bibitem{1868}
D.~Psaltis, A.~Sideris, and A.~Yamamura, ``A multilayered neural network
  controller,'' {\em IEEE Control Systems Magazine}, vol.~8, no.~2, pp.~17--21,
  1988.

\bibitem{parisini1995receding}
T.~Parisini and R.~Zoppoli, ``A receding-horizon regulator for nonlinear
  systems and a neural approximation,'' {\em Automatica}, vol.~31, no.~10,
  pp.~1443--1451, 1995.

\bibitem{Ioffe08}
A.~D. Ioffe, ``An invitation to tame optimization,'' {\em SIAM J. Optim.},
  vol.~19, no.~4, pp.~1894--1917, 2008.

\bibitem{clarke1975generalized}
F.~H. Clarke, ``Generalized gradients and applications,'' {\em Transactions of
  the American Mathematical Society}, vol.~205, pp.~247--262, 1975.

\bibitem{kahn1950random}
H.~Kahn, ``Random sampling (monte carlo) techniques in neutron attenuation
  problems--ii.,'' {\em Nucleonics}, vol.~6, no.~6, pp.~60--65, 1950.

\bibitem{robert1999monte}
C.~P. Robert, G.~Casella, and G.~Casella, {\em Monte Carlo statistical
  methods}, vol.~2.
\newblock Springer, 1999.

\bibitem{tokdar2010importance}
S.~T. Tokdar and R.~E. Kass, ``Importance sampling: a review,'' {\em Wiley
  Interdisciplinary Reviews: Computational Statistics}, vol.~2, no.~1,
  pp.~54--60, 2010.

\bibitem{bugallo2017adaptive}
M.~F. Bugallo, V.~Elvira, L.~Martino, D.~Luengo, J.~Miguez, and P.~M. Djuric,
  ``Adaptive importance sampling: The past, the present, and the future,'' {\em
  IEEE Signal Processing Magazine}, vol.~34, no.~4, pp.~60--79, 2017.

\bibitem{gittins2011multi}
J.~Gittins, K.~Glazebrook, and R.~Weber, {\em Multi-armed bandit allocation
  indices}.
\newblock John Wiley \& Sons, 2011.

\bibitem{lattimore2020bandit}
T.~Lattimore and C.~Szepesv{\'a}ri, {\em Bandit algorithms}.
\newblock Cambridge University Press, 2020.

\bibitem{lai1987adaptive}
T.~L. Lai, ``Adaptive treatment allocation and the multi-armed bandit
  problem,'' {\em The annals of statistics}, pp.~1091--1114, 1987.

\bibitem{AuerUCB}
P.~Auer, N.~Cesa-Bianchi, and P.~Fischer, ``Finite-time analysis of the
  multiarmed bandit problem,'' {\em Machine Learning}, vol.~47, no.~2,
  pp.~235--256, 2002.

\bibitem{berry1997bandit}
D.~A. Berry, R.~W. Chen, A.~Zame, D.~C. Heath, and L.~A. Shepp, ``Bandit
  problems with infinitely many arms,'' {\em The Annals of Statistics},
  vol.~25, no.~5, pp.~2103--2116, 1997.

\bibitem{wang2008algorithms}
Y.~Wang, J.-Y. Audibert, and R.~Munos, ``Algorithms for infinitely many-armed
  bandits,'' {\em Advances in Neural Information Processing Systems}, vol.~21,
  2008.

\bibitem{Kleinberg2008}
R.~Kleinberg, A.~Slivkins, and E.~Upfal, ``Multi-armed bandits in metric
  spaces,'' in {\em Proceedings of the Fortieth Annual ACM Symposium on Theory
  of Computing}, STOC '08, (New York, NY, USA), p.~681–690, Association for
  Computing Machinery, 2008.

\bibitem{JiaYuanYu}
S.~Bubeck, G.~Stoltz, and J.~Y. Yu, ``Lipschitz bandits without the lipschitz
  constant,'' in {\em Algorithmic Learning Theory} (J.~Kivinen,
  C.~Szepesv{\'a}ri, E.~Ukkonen, and T.~Zeugmann, eds.), (Berlin, Heidelberg),
  pp.~144--158, Springer Berlin Heidelberg, 2011.

\bibitem{NEURIPS2023_3b3a83a5}
E.~X.-Y. Gong and M.~Sellke, ``Asymptotically optimal quantile pure exploration
  for infinite-armed bandits,'' in {\em Advances in Neural Information
  Processing Systems} (A.~Oh, T.~Naumann, A.~Globerson, K.~Saenko, M.~Hardt,
  and S.~Levine, eds.), vol.~36, pp.~18551--18581, Curran Associates, Inc.,
  2023.

\bibitem{de2021bandits}
R.~De~Heide, J.~Cheshire, P.~M{\'e}nard, and A.~Carpentier, ``Bandits with many
  optimal arms,'' {\em Advances in Neural Information Processing Systems},
  vol.~34, pp.~22457--22469, 2021.

\bibitem{davis2020stochastic}
D.~Davis, D.~Drusvyatskiy, S.~Kakade, and J.~D. Lee, ``Stochastic subgradient
  method converges on tame functions,'' {\em Foundations of computational
  mathematics}, vol.~20, no.~1, pp.~119--154, 2020.

\bibitem{bolte2021conservative}
J.~Bolte and E.~Pauwels, ``Conservative set valued fields, automatic
  differentiation, stochastic gradient methods and deep learning,'' {\em
  Mathematical Programming}, vol.~188, pp.~19--51, 2021.

\bibitem{bolte2007clarke}
J.~Bolte, A.~Daniilidis, A.~Lewis, and M.~Shiota, ``Clarke subgradients of
  stratifiable functions,'' {\em SIAM Journal on Optimization}, vol.~18, no.~2,
  pp.~556--572, 2007.

\bibitem{borwein1997distinct}
J.~Borwein and X.~Wang, ``Distinct differentiable functions may share the same
  clarke subdifferential at all points,'' {\em Proceedings of the American
  Mathematical Society}, vol.~125, no.~3, pp.~807--813, 1997.

\bibitem{van1998tame}
L.~Van~den Dries, {\em Tame topology and o-minimal structures}, vol.~248.
\newblock Cambridge university press, 1998.

\bibitem{bolte2020mathematical}
J.~Bolte and E.~Pauwels, ``A mathematical model for automatic differentiation
  in machine learning,'' {\em Advances in Neural Information Processing
  Systems}, vol.~33, pp.~10809--10819, 2020.

\bibitem{NIPS2011_7e889fb7}
R.~Munos, ``Optimistic optimization of a deterministic function without the
  knowledge of its smoothness,'' in {\em Advances in Neural Information
  Processing Systems} (J.~Shawe-Taylor, R.~Zemel, P.~Bartlett, F.~Pereira, and
  K.~Weinberger, eds.), vol.~24, Curran Associates, Inc., 2011.

\bibitem{hochreiter1998vanishing}
S.~Hochreiter, ``The vanishing gradient problem during learning recurrent
  neural nets and problem solutions,'' {\em International Journal of
  Uncertainty, Fuzziness and Knowledge-Based Systems}, vol.~6, no.~02,
  pp.~107--116, 1998.

\bibitem{hanin2018neural}
B.~Hanin, ``Which neural net architectures give rise to exploding and vanishing
  gradients?,'' {\em Advances in neural information processing systems},
  vol.~31, 2018.

\bibitem{NEURIPS2020_9332c513}
W.~Moses and V.~Churavy, ``Instead of rewriting foreign code for machine
  learning, automatically synthesize fast gradients,'' in {\em Advances in
  Neural Information Processing Systems} (H.~Larochelle, M.~Ranzato,
  R.~Hadsell, M.~F. Balcan, and H.~Lin, eds.), vol.~33, pp.~12472--12485,
  Curran Associates, Inc., 2020.

\bibitem{ADAMTT}
M.~Aschenbrenner, L.~van~den Dries, and J.~van~der Hoeven, {\em Asymptotic
  differential algebra and model theory of transseries}, vol.~195 of {\em
  Annals of Mathematics Studies}.
\newblock Princeton University Press, Princeton, NJ, 2017.

\bibitem{rockafellar2009variational}
R.~T. Rockafellar and R.~J.-B. Wets, {\em Variational analysis}, vol.~317.
\newblock Springer Science \& Business Media, 2009.

\bibitem{fischer2005peano}
A.~Fischer, {\em Peano-differentiable functions in o-minimal structures}.
\newblock PhD thesis, Universit{\"a}t Passau, 2005.

\bibitem{Ross2013}
K.~A. Ross, {\em Elementary analysis}.
\newblock Undergraduate Texts in Mathematics, Springer, New York, second~ed.,
  2013.
\newblock The theory of calculus, In collaboration with Jorge M. L\'opez.

\bibitem{Clarke1990}
F.~H. Clarke, {\em Optimization and nonsmooth analysis}, vol.~5 of {\em
  Classics in Applied Mathematics}.
\newblock Society for Industrial and Applied Mathematics (SIAM), Philadelphia,
  PA, second~ed., 1990.

\bibitem{rademacher1919partielle}
H.~Rademacher, ``{\"U}ber partielle und totale differenzierbarkeit von
  funktionen mehrerer variabeln und {\"u}ber die transformation der
  doppelintegrale,'' {\em Mathematische Annalen}, vol.~79, no.~4, pp.~340--359,
  1919.

\bibitem{Warga81}
J.~Warga, ``Fat homeomorphisms and unbounded derivate containers,'' {\em J.
  Math. Anal. Appl.}, vol.~81, no.~2, pp.~545--560, 1981.

\end{thebibliography}

%%%%%%%%%%%%%%%%%%%%%%%%%%%%%%%%%%%%%%%%%%%%%%%%%%%%%%%%%%%%
\clearpage
\appendix

\section{Technical Appendices and Supplementary Material}\label{appendix_Supplementary_Material}
%Technical appendices with additional results, figures, graphs and proofs may be %submitted with the paper submission before the full submission deadline (see %above), or as a separate PDF in the ZIP file below before the supplementary %material deadline. There is no page limit for the technical appendices.

\noindent
In this appendix we give a proof of Theorem~\ref{thm:main} (Theorem~\ref{main_thm_proved} below) which is an o-minimal variant of~\cite[Theorem 1]{jordan2020exactly} (Theorem~\ref{thm:wrong} above). Theorem~\ref{thm:main} follows from the following nontrivial but well-known facts about functions $f:\R^n\to\R^m$:
\begin{enumerate}
\item if $f$ is definable in an o-minimal expansion of the real field (in short: $f$ \emph{is definable}), then $f$ is $C^1$ almost everywhere (cf. Lemma~\ref{diff_facts_for_definable_functions})
\item (Rademacher's Theorem) if $f$ is locally Lipschitz, then $f$ is Fr\'{e}chet-dfferentiable almost everywhere (cf. Lemma~\ref{Rademachers_theorem})
\item (Warga's Theorem) for locally Lipschitz $f$, the Clarke Jacobian $J_f^c$ of $f$ is ``blind'' to sets of measure zero (cf. Lemma~\ref{Clarke_Jacobian_blind_null})
\end{enumerate}

\noindent
Taking these facts for granted, here we otherwise give a self-contained account in order to see where the assumption of definability is relevant and where it is not. In particular, via a Path Lemma~\ref{path_lemma} we are able to avoid the Lebesgue integral in favor of the Riemann integral in the proof of~\cite[Theorem 1]{jordan2020exactly} (Corollary~\ref{Jordan_theorem_definable} below) in the case that $f$ is definable.

\medskip\noindent
We set the following assumptions/conventions throughout the appendix (more assumptions will be added along the way):
\begin{itemize}
\item $\frak{R}=(\R;<,+,\cdot,\ldots)$ is an o-minimal expansion of the real field and ``definable'' means ``definable in $\frak{R}$ with parameters'' 
\item we assume the reader is familiar with the concepts of \emph{o-minimality} and \emph{definability}; our main references for these concepts include~\cite{van1998tame} and~\cite[Appendix B]{ADAMTT}.
The reader is free to consider the special case where $\frak{R}=(\R;<,+,\cdot)$ is the real field, and thus ``definable'' means ``semialgebraic''
\item $f:\R^n\rightharpoonup\R^m$ is a partial function
\item $X\subseteq \operatorname{dom}(f)\subseteq\R^n$
\item for points $x,y\in\R^n$, we denote the line segment from $x$ to $y$ by:
\[
[x,y] \ := \ \{(1-t)x+ty:t\in[0,1]\}
\]
\item recall~\cite[Chapter 5]{rockafellar2009variational} that a \emph{set-valued map} $D:A\rightrightarrows B$ is by definition a function $D:A\to\mathcal{P}(B)$ (where $\mathcal{P}(B)$ denotes the powerset of $B$). When convenient, we identify $D$ with its \emph{graph} $\operatorname{gph}D:=\{(a,b):b\in D(a)\}\subseteq A\times B$. If $A\subseteq\R^m$ and $B\subseteq \R^n$, then we say $D$ is \emph{definable} if $\operatorname{gph}D\subseteq\R^{m+n}$ is definable.
\item we roughly follow the notation and definitions as in~\cite{jordan2020exactly}
\end{itemize}

\subsection*{Norms on vector spaces over $\R$} The concept of \emph{Lipschitz constant} is relative to a choice of norms on the spaces $\R^m,\R^n$. Here we recall some basic definitions and properties concerning norms.

\begin{definition}
Suppose $V$ is a vector space over $\R$. A \textbf{norm} $\|\cdot\|$ on $V$ is a function:
\[
\|\cdot\| \ : \ V\to [0,+\infty)
\]
which satisfies for all $x,y\in V$ and $\lambda\in\R$:
\begin{enumerate}
\item (Positive definiteness) $\|x\|=0$ iff $x=0$
\item (Absolute homogeneity) $\|\lambda x\|=|\lambda|\|x\|$
\item (Triangle inequality) $\|x+y\|\leq \|x\|+\|y\|$
\end{enumerate}
\end{definition}

\begin{lemma}\label{basic_norm_sup_facts}
Suppose $\|\cdot\|$ is a norm on $\R^n$. 
\begin{enumerate}
\item If $x_i$ is a sequence in $\R^n$ such that $x_i\to x$, then:
\[
\|x\| \ = \ \lim_{i\to\infty}\|x_i\| \ \leq \ \sup\{\|x_i\|:i\geq 0\}
\]
\item If $X\subseteq\R^n$, then:
\[
\sup_{x\in X}\|x\| \ = \ \sup_{x\in\operatorname{conv}(X)}\|x\|
\]
\end{enumerate}
\end{lemma}
\begin{proof}
(1) Follows from the fact that $\|\cdot\|:\R^n\to\R$ is a continuous function.

(2) The direction ``$\leq$'' is clear. For the ``$\geq$'' direction, suppose $x\in\operatorname{conv}(X)$, then there is $k\geq 1$ and $x_1,\ldots,x_k\in X$ and $\lambda_1,\ldots,\lambda_k\in[0,1]$ such that $x=\sum_{1\leq i\leq k}\lambda_ix_i$ and $\sum_{1\leq i\leq k}\lambda_i=1$. Suppose $i_0\in\{1,\ldots,k\}$ satisfies $\|x_{i_0}\|=\max\{\|x_i\|:i=1,\ldots,k\}$. Then:
\[
\textstyle \|\sum_{1\leq i\leq k}\lambda_ix_i\| \ \leq \ \sum_{1\leq i\leq k}\lambda_i\|x_i\| \ \leq \ \sum_{1\leq i\leq k}\lambda_i\|x_{i_0}\| \ = \ \|x_{i_0}\|\qedhere
\]
\end{proof}

\noindent
The following lemma includes a template for the type of construction which occurs in the definition of \emph{Clarke Jacobian}:

\begin{lemma}\label{sup_equivalence_template}
Given $S\subseteq X$ and $J:S\to\R^m$, define:
\[
D_J \ : \ X\rightrightarrows\R^m,\quad x \ \mapsto \ D_J(x) \ := \ \{\lim_{i\to\infty}J(x_i):x_i\to x, x_i\in S\}
\]
Then:
\[
\sup_{x\in S}\|J(x)\|_{\beta} \ = \ \sup_{x\in X, G\in \operatorname{conv}(D_J(x))}\|G\|_{\beta}
\]
\end{lemma}
\begin{proof}
It is clear that the direction ``$\leq$'' holds. Conversely, first note that by Lemma~\ref{basic_norm_sup_facts}(2) we have:
\[
\sup_{x\in X, G\in \operatorname{conv}(D_J(x))}\|G\|_{\beta} \ = \ \sup_{x\in X}\sup_{G\in\operatorname{conv}(D_J(x))}\|G\|_{\beta} \ = \ \sup_{x\in X}\sup_{G\in D_J(x)}\|G\|_{\beta}
\]
Next, let $x\in X$ and $G\in D_J(x)$ be arbitrary, and take $x_i\in S$ such that $x_i\to x$ and $J(x_i)\to G$. Then by Lemma~\ref{basic_norm_sup_facts}(1) we have
\[
\|G\|_{\beta} \ \leq \ \sup\{\|J(x_i)\|_{\beta}:i\geq 0\} \ \leq \ \sup_{x\in S}\|J(x)\|_{\beta}
\]
Since $x,G$ were arbitrary, this yields the ``$\geq$'' direction.
\end{proof}

\noindent
As is common, we shall identify $\R^n$ with its dual space. A norm $\|\cdot\|$ on $\R^n$ induces a \textbf{dual norm} $\|\cdot\|_*$ on the dual space $\R^n$:
\[
\|\cdot\|_* \ : \ \R^n\to[0,+\infty),\quad y \ \mapsto \ \|y\|_* \ := \ \sup_{\|x\|\leq 1}\langle x,y\rangle
\]
In our finite-dimensional setting we have $\|\cdot\|_{\ast\ast}=\|\cdot\|$. We will also make use H\"{o}lder's inequality for dual norms:

\begin{lemma}\cite[Proposition 1]{jordan2020exactly}\label{Holder_inequality_dual_norm}
Suppose $\|\cdot\|$ is a norm on $\R^n$. Then for every $x,y\in\R^n$ we have:
\[
\langle x,y\rangle \ \leq \ \|x\|\cdot\|y\|_*
\]
\end{lemma}

\noindent
Using the double dual norm and H\"{o}lder's inequality, we get:

\begin{lemma}[Continuous triangle inequality]\label{norm_integral_triangle_inequality}
Suppose $g:[0,1]\to\R^m$ is Riemann integrable and $\|\cdot\|$ is a norm on $\R^m$. Then $\|g\|:[0,1]\to\R$ is Riemann integrable and:
\[
\textstyle \|\int_0^1g(t)dt\| \ \leq \ \int_0^1\|g(t)\|dt
\]
\end{lemma}
\begin{proof}
The first statement is clear. Using the double dual norm and linearity, it suffices to show:
\[
\textstyle \|\int_0^1g(t)dt\| \ = \ \sup_{\|y\|_*\leq 1}\langle y, \int_0^1g(t)dt\rangle \ = \ \sup_{\|y\|_*\leq 1}\int_0^1\langle y, g(t)\rangle dt \ \leq \ \int_0^1\|g(t)\|dt
\]
Let $y$ be arbitrary such that $\|y\|_*\leq 1$, then it suffices to show:
\[
\textstyle \int_0^1\langle y,g(t)\rangle dt \ \leq \ \int_0^1\|g(t)\|dt
\]
For this, it suffices to show for every $t\in[0,1]$:
\[
\langle y,g(t)\rangle \ \leq \ \|g(t)\|
\]
However this follows from H\"{o}lder's inequality~\ref{Holder_inequality_dual_norm} for dual norms:
\[
\textstyle \langle y, g(t)\rangle \ \leq \ \|y\|_*\|g(t)\| \ \leq \ \|g(t)\| 
\]
since $\|y\|_*\leq 1$. 
\end{proof}

For the rest of this appendix, we further assume:
\begin{itemize}
\item $\|\cdot\|_{\alpha}$ and $\|\cdot\|_{\beta}$ are norms on $\R^n$ and $\R^m$ respectively (we don't require these norms to be definable, e.g., the graph of the function $\|\cdot\|_{\alpha}:\R^n\to\R$, as a subset of $\R^{n+1}$, is not assumed to be definable in $\frak{R}$, likewise for $\|\cdot\|_{\beta}$)
\end{itemize}

We may also define the \textbf{matrix norm} $\|\cdot\|_{\alpha,\beta}$ on $\R^{m\times n}$ induced by $\|\cdot\|_{\alpha}$ and $\|\cdot\|_{\beta}$:
\[
\|\cdot\|_{\alpha,\beta} \ : \ \R^{m\times n}\to[0,+\infty),\quad A \ \mapsto \ \sup_{\|x\|_{\alpha}\leq 1}\|Ax\|_{\beta}
\]
There is also an analogous H\"{o}lder-type inequality for matrix norms:

\begin{lemma}\cite[Proposition A.2]{jordan2020exactly}\label{Jordan_Prop_A_2}
For every $A\in\R^{m\times n}$ and $x\in\R^n$ we have:
\[
\|Ax\|_{\beta} \ \leq \ \|A\|_{\alpha,\beta}\|x\|_{\alpha}
\]
\end{lemma}
\iffalse
\begin{proof}
We may assume that $x$ is nonzero and set $y:=x/\|x\|_{\alpha}$. Since $\|y\|_{\alpha}=1$, we have:
\[
\|Ax\|_{\beta} \ = \ \|x\|_{\alpha}\|Ay\|_{\beta} \ \leq \ \sup_{\|y'\|_{\alpha}\leq 1}\|x\|_{\alpha}\|Ay'\|_{\beta} \ = \ \|x\|_{\alpha}\|A\|_{\alpha,\beta} \qedhere
\]
\end{proof}
\fi

\subsection*{The Lipschitz property}

\begin{definition}\cite[Definition 1]{jordan2020exactly}
The \textbf{local $(\alpha,\beta)$-Lipschitz constant} of  $f$ over  $X$ is the (possibly infinite) quantity:
\[
L^{(\alpha,\beta)}(f,X) \ := \ \sup_{x,y\in X}\frac{\|f(x)-f(y)\|_{\beta}}{\|x-y\|_{\alpha}} \quad (x\neq y)
\]
Moreover, if $L^{(\alpha,\beta)}(f,X)$ is finite, we say that $f$ is \textbf{$(\alpha,\beta)$-Lipschitz} over $X$.
We say that $f$ is \textbf{locally $(\alpha,\beta)$-Lipschitz} over $X$ if for every $x\in X$, there exists a neighbourhood $U$ of $x$ such that $f$ is $(\alpha,\beta)$-Lipschitz over $X\cap U$.
\end{definition}

\noindent
The following is well-known:

\begin{lemma}\label{locally_lipschitz_compact}
If $X$ is compact and $f$ is locally $(\alpha,\beta)$-Lipschitz over $X$, then $f$ is $(\alpha,\beta)$-Lipschitz over $X$.
\end{lemma}

\subsection*{Directionally differentiable, Fr\'{e}chet differentiable, and $C^1$} For the rest of this appendix, we further assume:
\begin{itemize}
\item $X$ is open
\end{itemize}

\begin{definition}
Given $x\in X$ and $v\in \R^n$, we say that $f$ is \textbf{directionally differentiable} at $x$ in the direction $v$ if there exists $\ell\in\R^m$ such that:
\[
\lim_{t\downarrow0}(f(x+tv)-f(x))/t \ = \ \ell
\]
in which case we define the \textbf{directional derivative} $f'(x;v)$ of $f$ at $x$ in the direction $v$ to be $f'(x;v):=\ell$. We say that $f$ is \textbf{directionally differentiable} at $x$ if $f$ is directionally differentiable at $x$ in the direction $v$ for every $v\in\R^n$.
\end{definition}

\noindent
Here is the basic relationship between a directional derivative and the Lipschitz constant:

\begin{lemma}\label{lipschitz_constant_bounds_dir_der}
If $f$ is directionally differentiable at $x$ in the direction $v$, then:
\[
\|f'(x;v)\|_{\beta} \ \leq \ L^{(\alpha,\beta)}(f,X)\|v\|_{\alpha}
\]
\end{lemma}
\begin{proof}
We may suppose $v\neq 0$. For every $\varepsilon>0$ we can find $t>0$ sufficiently small such that $\|f'(x;v)\|_{\beta}$ is within $\varepsilon$ from the quantity:
\[
\frac{\|f(x+tv)-f(x)\|_{\beta}}{t} \ = \ \|v\|_{\alpha}\frac{\|f(x+tv)-f(x)\|_{\beta}}{\|(x+tv)-x\|_{\alpha}}
\]
Since the quantity on the righthand side contributes to the definition of $L^{(\alpha,\beta)}(f,X)\|v\|_{\alpha}$ as a certain supremum, the inequality follows.
\end{proof}

\begin{definition}
We say that $f$ is \textbf{Fr\'{e}chet-differentiable} at $x$ if there exists a linear map $L:\R^n\to\R^m$ such that $f(x)+L(y-x)$ is a first-order approximation of $f$ at $x$, i.e., we have:
\[
\lim_{y\to x}\frac{\|f(y)-f(x)-L(y-x)\|_{\beta}}{\|y-x\|_{\alpha}} \ = \ 0
\]
in which case we define the \textbf{Jacobian} $J_f(x)$ of $f$ at $x$ to be the (necessarily unique) linear map $J_f(x)=L$; we shall identify the Jacobian $J_f(x)$ with a matrix in $\R^{m\times n}$ which represents $L$ with respect to the standard basis.
\end{definition}

\begin{lemma}\label{Jacobian_gives_directional_derivative}
If $f$ is Fr\'{e}chet-differentiable at $x$, then $f$ is directionally differentiable at $x$ and for every $v\in\R^n$, we have:
\[
J_f(x)v \ = \ f'(x;v)
\]
\end{lemma}

Here is a relationship between the Jacobian at a point and the Lipschitz constant:

\begin{lemma}\label{Easy_Lemma_Lipschitz_constant_bounds_Jacobian}
If $f$ is Fr\'{e}chet-differentiable at $z\in X$, then:
\[
\|J_f(z)\|_{\alpha,\beta} \ \leq \ L^{(\alpha,\beta)}(f,X)
\]
\end{lemma}
\begin{proof}
By Lemmas~\ref{Jacobian_gives_directional_derivative} and~\ref{lipschitz_constant_bounds_dir_der} we have:
\[
\|J_f(z)\|_{\alpha,\beta} \ = \ \sup_{\|v\|_{\alpha}\leq 1}\|J_f(z)v\|_{\beta} \ = \ \sup_{\|v\|_{\alpha}\leq 1}\|f'(z;v)\|_{\beta} \ \leq \ \sup_{\|v\|_{\alpha}\leq1}L^{(\alpha,\beta)}(f,X)\|v\|_{\alpha} \ = \ L^{(\alpha,\beta)}(f,X) \qedhere
\]
\end{proof}

\begin{definition}
Given $x\in X$, we say that $f$ is $C^1$ at $x$ if there exists an open neighbourhood $U$ of $x$ in $X$ such that $f$ is Fr\'{e}chet-differentiable at each $y\in U$ and moreover, the function $y\mapsto J_f(y):U\to\R^{m\times n}$ is continuous.
\end{definition}

We let $\operatorname{Diff}(f)\subseteq X$ denote the set of points $x\in X$ such that $f$ is Fr\'{e}chet-differentiable at $x$ and we let $\operatorname{Diff}^1(f)$ denote the set of points $x\in X$ such that $f$ is $C^1$ at $x$. Clearly $\operatorname{Diff}^1(f)\subseteq\operatorname{Diff}(f)$.

\medskip\noindent
Here is the main fact we will use about definable functions:

\begin{lemma}\label{diff_facts_for_definable_functions}
If $f:X\to\R^m$ is definable, then:
\begin{enumerate}
\item $\operatorname{Diff}(f)$ and $\operatorname{Diff}^1(f)$ are definable
\item $\dim(X\setminus\operatorname{Diff}^1(f))<n$, and thus $X\setminus\operatorname{Diff}^1(f)$ is nowhere dense and has Lebesgue measure zero
\end{enumerate}
\end{lemma}
\begin{proof}[Remark about proof] (1) is an easy exercise in definability. (2) is nontrivial and follows from \emph{Smooth Cell Decomposition}~\cite[7.3.2]{van1998tame}. Here the dimension $\dim$ is taken in the sense of definable sets in an o-minimal structure~\cite[4.1]{van1998tame} which agrees with and generalizes the usual notion of \emph{dimension} for $C^1$ manifolds, at least in the case that the manifold is presented as an embedded submanifold of $\R^n$ and the underlying set of the manifold is a definable subset of $\R^n$.
The claim about Lebesgue measure follows from the fact that for any connected embedded $C^1$ submanifold $M\subseteq\R^n$ (definable or not), if $\dim M<n$, then $M$ has Lebesgue measure zero.
\end{proof}

\subsection*{A supremum of Jacobians bound a difference quotient}

\noindent
The lemma here is routine although we show how to use the Riemann integral instead of the Lebesgue integral when we are in the definable setting. Here is the setup:
\begin{itemize}
\item Fix distinct points $x,y\in\R^n$ such that $[x,y]\subseteq X$
\item let $L:=[x,y]\setminus\{x,y\}$ be the ``open'' line segment from $x$ to $y$
\item $f:X\to\R^m$ is  $(\alpha,\beta)$-Lipschitz
\item $f$ is either $C^1$ on $L$, or definable and Fr\'{e}chet-differentiable on $L$
\end{itemize}

\begin{lemma}\label{basic_lemma_Jacobian_bounds_diff_quot}
In the above setup we have:
\[
\|f(y)-f(x)\|_{\beta} \ \leq \ \sup_{z\in L}\|J_f(z)\|_{\alpha,\beta}\cdot\|y-x\|_{\alpha}
\]
\end{lemma}
\begin{proof}
First define the function:
\[
h \ : \ [0,1]\to\R^m,\quad t \ \mapsto \ f((1-t)x+ty)
\]
We know that:
\begin{itemize}
\item $h$ is continuous
\item $h$ is differentiable on $(0,1)$ with derivative $h'(t)=J_f((1-t)x+ty)(y-x)$, computed via the chain-rule for Fr\'{e}chet-differentiable functions 
\item in particular, by Lemma~\ref{Easy_Lemma_Lipschitz_constant_bounds_Jacobian}, $h'$ is bounded because $f$ is Lipschitz over the compact set $[x,y]$; c.f. Lemma~\ref{locally_lipschitz_compact}
\item on $(0,1)$, $h'$ is continuous at all but finitely many points: by assumption either $f$ is $C^1$, or if $f$ is definable then this follows by the \emph{Monotonicity Theorem}~\cite[3.1.2]{van1998tame} (in fact, $h'$ is actually continuous on $(0,1)$ by~\cite[5.7]{fischer2005peano})
\end{itemize}
Next, define the function:
\[
g \ : \ [0,1]\to\R^m,\quad t \ \mapsto \ g(t) \ := \ \begin{cases}
J_f((1-t)x+ty)(y-x) & \text{if $t\in(0,1)$} \\
0 & \text{if $t=0,1$}
\end{cases}
\]
Then $g$ is bounded and continuous at all but finitely many points, hence $g$ is Riemann integrable.
Moreover, $h'(t)=g(t)$ on $(0,1)$ and so by the \emph{Fundamental Theorem of Calculus}~\cite[34.1]{Ross2013} we have:
\[
\textstyle f(y)-f(x) \ = \ h(1)-h(0) \ = \ \int_0^1g(t)dt
\]
Next, note that we have the following bound on $g(t)$, for $t\in (0,1)$, by Lemma~\ref{Jordan_Prop_A_2}:
\[
\|g(t)\|_{\beta} \ = \ \|J_f((1-t)x+ty)(y-x)\|_{\beta} \ \leq \ \|J_f((1-t)x+ty)\|_{\alpha,\beta}\|y-x\|_{\alpha} \ \leq \ \sup_{z\in L}\|J_f(z)\|_{\alpha,\beta}\cdot\|y-x\|_{\alpha}
\]
Moreover, the overall bound holds for all $t\in[0,1]$.
The main inequality now proceeds as follows:
\begin{align*}
\|f(y)-f(x)\|_{\beta} \ &= \ \textstyle \|\int_0^1g(t)dt\|_{\beta} \\
&\leq \ \textstyle \int_0^1\|g(t)\|_{\beta}dt \quad\text{by Lemma~\ref{norm_integral_triangle_inequality}}\\
&\leq \ \textstyle \int_0^1\sup_{z\in L}\|J_f(z)\|_{\alpha,\beta}\cdot\|y-x\|_{\alpha}dt \\
&= \ \sup_{z\in L}\|J_f(z)\|_{\alpha,\beta}\cdot\|y-x\|_{\alpha} \qedhere
\end{align*}
\end{proof}

\noindent
The following is now immediate from Lemmas~\ref{Easy_Lemma_Lipschitz_constant_bounds_Jacobian} and~\ref{basic_lemma_Jacobian_bounds_diff_quot}:

\begin{cor}[Lemma~\ref{lipschitz_constant_formula_smooth}]\label{cor_lipschitz_constant_formula_smooth}
If $f:\R^n\to\R^m$ is $C^1$ and $(\alpha,\beta)$-Lipschitz over an open convex set $X$, then:
\[
L^{(\alpha,\beta)}(f,X) \ = \ \sup_{x\in X}\| J_f(x)\|_{\alpha,\beta}
\]
\end{cor}
In the case that $f$ is definable, then we may weaken the $C^1$ assumption in Corollary~\ref{cor_lipschitz_constant_formula_smooth} to Fr\'{e}chet-differentiable and the same argument works; although in this case, this statement will be superseded by Lemma~\ref{main_theorem_special_case} below.

\subsection*{The Clarke Jacobian} In this subsection we assume:
\begin{itemize}
\item $f:X\to\R^m$ is locally $(\alpha,\beta)$-Lipschitz
\end{itemize}

\begin{definition}\label{def_Clarke_Jacobian}\cite[\S2.6]{Clarke1990}
Define the \textbf{Clarke Jacobian} $J_f^c$ of $f$ over $X$ to be the set-valued map:
\[
J_f^c \ : \ X\rightrightarrows\R^{m\times n},\quad z \ \mapsto \ J_f^c(z) \ := \ \operatorname{conv}\{\lim_{j\to\infty}J_f(z_j):z_j\to z, z_j\in \operatorname{Diff}(f)\}
\]
\end{definition}

\noindent
The following is immediate from the definition:

\begin{lemma}\label{Clarke_Jacobian_special_case}
If $x\in\operatorname{Diff}(f)$, then $\{J_f(x)\}\subseteq J_f^c(x)$; moreover, if $f$ is $C^1$ at $x$ then equality holds.
\end{lemma}

\noindent
Note that if $f$ is definable, then Lemma~\ref{diff_facts_for_definable_functions} guarantees that $\dim(X\setminus\operatorname{Diff}(X))<n$ and thus $X\setminus\operatorname{Diff}(f)$ has Lebesgue measure zero. The general case is handled by:

\begin{lemma}[Rademacher]\cite{rademacher1919partielle}\label{Rademachers_theorem}
The set $X\setminus\operatorname{Diff}(f)\subseteq\R^n$ has Lebesgue measure zero.
\end{lemma}

\begin{lemma}
For every $x\in X$, $J_f^c(x)$ is nonempty, compact, and convex.
\end{lemma}
\begin{proof}
Convexity is clear and compactness follows from the fact that lipschitz functions have bounded Jacobian (Lemma~\ref{Easy_Lemma_Lipschitz_constant_bounds_Jacobian}).  Nonemptiness follows from Rademacher's Theorem.
\end{proof}

It is also convenient to consider the following ``variant'' of $J_f^c$; suppose $N\subseteq X$ has Lebesgue measure zero and define:
\[
J_f^{c,N} \ : \ X\rightrightarrows\R^{m\times n},\quad z \ \mapsto \ J_f^{c,N}(z) \ := \ \operatorname{conv}\{\lim_{j\to\infty}J_f^c(z_j):z_j\to z, z_j\in \operatorname{Diff}(f)\setminus N\}
\]

\begin{lemma}[Warga]\cite[Theorem 4]{Warga81}\label{Clarke_Jacobian_blind_null}
For any set $N\subseteq X$ of Lebesgue measure zero, we have $J_f^{c,N}=J_f^c$.
\end{lemma}

\begin{exercise}
Assuming $f$ and $N$ are definable, give a proof of Lemma~\ref{Clarke_Jacobian_blind_null} which does not use measure theory (use the fact that $\dim N<n$).
\end{exercise}

\subsection*{The path lemma}

Here is the setup:
\begin{itemize}
\item $X\subseteq\R^n$ is a definable open convex set 
\end{itemize}

\begin{lemma}[Path Lemma]\label{path_lemma}
Suppose $x,y\in X$ and $B\subseteq\R^n$ is a definable set such that $\dim B<n$.
For every $\varepsilon>0$, there exists $k\geq 0$ and points $x_0=x,x_1,\ldots,x_k,x_{k+1}=y$ in $X$ such that:
\begin{itemize}
\item $\sum_{0\leq i\leq k}\|x_{i+1}-x_{i}\|_{\alpha}-\|y-x\|_{\alpha}\leq\varepsilon$, and
\item $\{(1-t)x_i+tx_{i+1}:t\in(0,1)\}\cap B=\varnothing$ for each $i=0,\ldots,k$
\end{itemize}
\end{lemma}
\begin{proof}
Let $v=(x+y)/2$ denote the midpoint between $x$ and $y$, and consider the definable set:
\[
P \ := \ \{v+w:\langle w,y-x\rangle=0\}\cap X \ \subseteq \ \R^n
\]
So $P$ is the intersection of $X$ with the affine space passing through the midpoint $v$ and orthogonal to the segment $y-x$; it follows $\dim P=n-1$ since $X$ is open. We want to show that there are enough points of $P$ which are arbitrarily close to $v$ such that the pair of line segments $[x,z]$ and $[z,y]$ each have finite intersection with $B$. For this, we can consider the ``bad'' points of $P$:
\[
P_1 \ := \ \{z\in P:\dim(([x,z]\cup[z,y])\cap B)=1\}
\]
Then $P_1$ is a definable set by \emph{Definability of dimension}~\cite[4.1.5]{van1998tame}; it suffices to show that $\dim P_1<n-1$. For this, first trim the set $B$ down:
\[
B' \ := \ B\cap \bigcup_{z\in P_1}[x,z]\cup [z,y]
\]
and consider the definable surjection:
\[
f \ : \ B'\to P_1, \quad b \ \mapsto \ f(b) \ := \ \text{the unique $z\in P_1$ such that $b\in [x,z]\cup [z,y]$}
\]
Note that by construction we have:
\[
P_1 \ = \ \{z\in P_1:\dim f^{-1}(z)=1\}
\]
and thus by the \emph{Dimension Formula}~\cite[4.1.6(ii)]{van1998tame} we have:
\[
1+\dim P_1 \ = \ \dim f^{-1}[P_1] \ = \ \dim B' \ < \ n
\]
and thus $\dim P_1<n-1$.
\end{proof}

Note: the lemma is still true without assuming that $X$ is definable.

\subsection*{A special case of the main theorem}

\begin{lemma}\label{main_theorem_special_case}
If $f:X\to\R^m$ is definable, then:
\[
L^{(\alpha,\beta)}(f,X) \ = \ \sup_{x\in \operatorname{Diff}(f)}\|J_f(x)\|_{\alpha,\beta}
\]
\end{lemma}
\begin{proof}
Set $L_0:=L^{^{(\alpha,\beta)}}(f,X)$ and $L_1:=\sup_{x\in\operatorname{Diff}(f)}\|J_f(x)\|_{\alpha,\beta}$. The inequality $L_0\geq L_1$ is clear by Lemma~\ref{Easy_Lemma_Lipschitz_constant_bounds_Jacobian}. To show $L_0\leq L_1$, we may assume that $L_1$ is finite. Let $x,y\in X$ be arbitrary distinct points and let $\varepsilon>0$ be arbitrary; it suffices to show:
\[
\frac{\|f(y)-f(x)\|_{\beta}}{\|y-x\|_{\alpha}} \ \leq \ L_1+\varepsilon
\]
Set $\varepsilon':=\varepsilon\|y-x\|_{\alpha}/L_1$ and apply the Path Lemma~\ref{path_lemma} to obtain $k\geq 0$, points $x_0=x,x_1,\ldots,x_k,x_{k+1}=y$ in $X$ such that:
\begin{itemize}
\item $\sum_{0\leq i\leq k}\|x_{i+1}-x_i\|_{\alpha}-\|y-x\|_{\alpha}<\varepsilon'$, and
\item $\{(1-t)x_i+tx_{i+1}:t\in (0,1)\}\cap B=\varnothing$ for each $i=0,\ldots,k$, where $B:=X\setminus \operatorname{Diff}(f)$; note that $\dim B<n$.
\end{itemize}
Then we have:
\begin{align*}
\|f(y)-f(x)\|_{\beta} \ &\leq \ \sum_{0\leq i\leq k}\|f(x_{i+1})-f(x_i)\|_{\beta} \quad\text{by Triangle inequality}\\
&\leq \ \sum_{0\leq i\leq k}L_1\|x_{i+1}-x_i\|_{\alpha} \quad\text{by Lemma~\ref{basic_lemma_Jacobian_bounds_diff_quot}}\\
&< \ L_1\|y-x\|_{\alpha}+L_1\varepsilon' \quad\text{by choice of $x_i$'s}
\end{align*}
Thus:
\[
\frac{\|f(y)-f(x)\|_{\beta}}{\|y-x\|_{\alpha}} \ < \ L_1+\frac{L_1\varepsilon'}{\|y-x\|_{\alpha}} \ = \ L_1+\varepsilon \qedhere
\]
\end{proof}

\noindent
As an application, can now prove Theorem~\ref{thm:wrong} in the case that $f$ is definable:

\begin{cor}\label{Jordan_theorem_definable}
If $f$ is definable, then:
\[
L^{(\alpha,\beta)}(f,X) \ = \ \sup_{x\in X,G\in J_f^c(x)}\|G\|_{\alpha,\beta}
\]
\end{cor}
\begin{proof}
Immediately follows from Lemmas~\ref{main_theorem_special_case} and~\ref{sup_equivalence_template}.
\end{proof}

\subsection*{Good maps and the equivalence lemma} In this subsection, to simplify notation, given a set-valued map $D:X\rightrightarrows\R^{m\times n}$, we set:
\[
s(D) \ := \ \sup_{x\in X,G\in D(x)}\|G\|_{\alpha,\beta}
\]

\begin{definition}
We say a set-valued map $D:X\rightrightarrows\R^{m\times n}$ is \textbf{good} for $f$ if:
\begin{itemize}
\item $D(x)\subseteq J_f^c(x)$ for every $x\in X$, and
\item $D(x)$ is nonempty for almost every $x\in X$
\end{itemize}
\end{definition}

\begin{example}\label{good_map_examples}
Here are the two main examples of good maps for $f$:
\begin{itemize}
\item the full Clarke Jacobian $J_f^c:X\rightrightarrows\R^{m\times n}$ is good for $f$
\item the usual Jacobian:
\[
J_f \ : \ X\rightrightarrows\R^{m\times n},\quad x \ \mapsto \ J_f(x) \ := \ \begin{cases}
\{J_f(x)\} & \text{if $x\in\operatorname{Diff}(f)$} \\
\varnothing & \text{otherwise}
\end{cases}
\]
is also good for $f$, by Rademacher's Theorem~\ref{Rademachers_theorem}
\end{itemize}
\end{example}

\medskip\noindent
If $D_0,D_1$ are good maps for $f$, then we define:
\[
D_0 \ \leq \ D_1 \quad :\Longleftrightarrow \quad D_0(x) \ \subseteq \ D_1(x)\quad\text{for every $x\in X$}
\]
The binary relation $\leq$ endows the collection of all good maps for $f$ with the structure of a \emph{partial order}; in fact, this partial order is always a \emph{join-semilattice} (with join operation given by the union $D_0\vee D_1:=D_0\cup D_1$). The following is obvious:

\begin{lemma}\label{good_maps_trivial_inequality}
If $D_0\leq D_1$ are good maps for $f$, then $s(D_0)\leq s(D_1)$.
\end{lemma}

\noindent
Under the assumption of definability, good maps have the following property:

\begin{lemma}
If $f$ is definable and $D$ is good for $f$, then $D(x)=\{J_f(x)\}$ for almost every $x\in X$.
\end{lemma}
\begin{proof}
Since $f$ is $C^1$ at $x$ for almost every $x\in X$ by Lemma~\ref{diff_facts_for_definable_functions}, it follows that $J_f^c(x)=\{J_f(x)\}$ for almost every $x\in X$ by Lemma~\ref{Clarke_Jacobian_special_case}, and thus $D(x)=\{J_f(x)\}$ for almost every $x\in X$ by the definition of \emph{good map}.
\end{proof}

\noindent
It follows immediately that when $f$ is definable, then the partial order $\leq$ is a lattice:

\begin{lemma}\label{good_maps_form_lattice}
If $f$ is definable and $D_0,D_1$ are good maps for $f$, then $D:=D_0\cap D_1$ is a good map for $f$.
\end{lemma}

\begin{lemma}[Equivalence lemma]\label{equivalence_lemma}
If $f$ is definable and $D_0,D_1$ are two set-valued maps which are good for $f$, then $s(D_0)=s(D_1)$.
\end{lemma}
\begin{proof}
It suffices to show that $s(D)=s(J_f^c)$, where $D$ is an arbitrary good map for $f$. Define $D':=D\cap\{J_f(x)\}$, which is also good map for $f$ by Lemma~\ref{good_maps_form_lattice} and Example~\ref{good_map_examples}; it suffices to show $s(D')=s(J_f^c)$, by Lemma~\ref{good_maps_trivial_inequality}. To set notation, consider:
\begin{itemize}
\item $S:=\operatorname{domain}(D')=\{x\in X:D'(x)\neq\varnothing\}$, so $S\subseteq\operatorname{Diff}(f)$
\item $N:=X\setminus S$, so $N$ has Lebesgue measure zero
\item $J:=J_f|_S$; thus $J:S\to\R^{m\times n}$ fits the setup of Lemma~\ref{sup_equivalence_template}
\item thus we may further define $D_J:X\rightrightarrows\R^{m\times n}$ to be $D_J(x):=\{\lim_{i\to\infty}J(x_i):x_i\to x,x_i\in S\}$
\end{itemize}
Then:
\begin{align*}
s(D') \ &= \ \sup_{x\in S}\|J(x)\|_{\alpha,\beta} \quad\text{by definition} \\ 
&= \ \sup_{x\in X, G\in\operatorname{conv}(D_J(x))}\|G\|_{\alpha,\beta} \quad\text{by Lemma~\ref{sup_equivalence_template}} \\ 
&= \ s(J_f^{c,N}) \quad\text{by definition} \\ 
&= \ s(J_f^c) \quad\text{by Lemma~\ref{Clarke_Jacobian_blind_null}} \qedhere
\end{align*}
\end{proof}

\subsection*{The main theorem}

We may now prove the main Theorem~\ref{thm:main}:

\begin{thm}\label{main_thm_proved}
If $f$ is definable and $D$ is good for $f$, then:
\[
L^{(\alpha,\beta)}(f,X) \ = \ \sup_{x\in X, G\in D(x)}\|G\|_{\alpha,\beta}
\]
\end{thm}
\begin{proof}
By the Equivalence Lemma~\ref{equivalence_lemma} it suffices to prove the stated equality for at least one $D$ which is good for $f$. 
Let $D$ be the good map $\{J_f(x)\}$ (cf. Example~\ref{good_map_examples});
then the equality holds for this $D$ by Lemma~\ref{main_theorem_special_case}.
\end{proof}

\end{document}